\newif\ifllncs  %
\newif\iffull   %

\fulltrue

\def\shownotes{1}  %

\ifnum\shownotes=1
\newcommand{\authnote}[2]{\textcolor{red}{\textsf{#1 }\textcolor{blue}{ #2}}}
\else
\newcommand{\authnote}[2]{}
\fi

\ifllncs
  \documentclass[runningheads,a4paper]{llncs}

\else
\documentclass[11pt]{article}
\fi

\usepackage
[pdftex,colorlinks=true,pdfpagemode=UseNone,urlcolor=black,linkcolor=black,citecolor=black,pdfstartview=FitH]{hyperref}
\usepackage{amsmath,amsfonts, amssymb}
\usepackage{graphicx}
\usepackage{amsthm}
\usepackage{algorithm}
\usepackage{cleveref}
\usepackage{algpseudocode}
\usepackage{tikz}
\usepackage{url}
\usepackage{caption}
\usepackage{subcaption}
\usepackage{booktabs}
\usepackage{diagbox}

\iffull
\fi

\ifllncs\else

\newtheorem{theorem}{Theorem}[section]
\newtheorem{lemma}[theorem]{Lemma}
\newtheorem{definition}[theorem]{Definition}

\fi

\newtheorem{setting}[theorem]{Setting}

\newtheorem{Empirical}{Empirical Conclusion}

\newcommand{\LPN}{\mathsf{LPN}}

\newcommand{\la}{\leftarrow}

\newcommand{\pr}{\mbox{Prob}}

\newcommand{\ipd}[2]{\left \langle {#1}, {#2} \right \rangle}

\newcommand{\mat}[1] { \mathbf{#1} }		%
\newcommand{\ary}[1] { \mathbf{#1} }		%

\newcommand{\N}{\mathbb{N}}

\newcommand{\R}{\mathbb{R}}
\newcommand{\set}[1]{ \left\{ #1 \right\}  }   %

\newcommand{\Z}{\mathbb{Z}}

\newcommand{\HW}{\mathsf{HW}}

\newcommand{\distribution}{\mathcal{D}}
\newcommand{\datadistribution}{\distribution_{\mathrm{data}}}
\newcommand{\functionclass}{\mathcal{H}}
\newcommand{\dataset}{\mathcal{S}}
\newcommand{\trainingset}{\mathcal{S}_{\mathrm{Train}}}
\newcommand{\boostingset}{\mathcal{S}_{\mathrm{Boost}}}
\newcommand{\trainset}{\trainingset}
\newcommand{\testset}{\mathcal{S}_{\mathrm{Test}}}

\newcommand{\populoss}{\mathcal{L}_{\distribution}}
\newcommand{\trainloss}{\mathcal{L}_{\mathrm{Train}}}
\newcommand{\testloss}{\mathcal{L}_{\mathrm{Test}}}
\newcommand{\batchloss}{\mathcal{L}_{\mathrm{Batch}}}
\newcommand{\algo}{\mathcal{A}}
\newcommand{\E}{\mathrm{E}}
\newcommand{\nonneg}{\R^{\ge 0}}
\newcommand{\argmin}{\arg \min}
\newcommand{\relu}{\mathrm{ReLU}}
\newcommand{\sigmoid}{\mathrm{Sigmoid}}
\newcommand{\cosine}{\mathrm{Cos}}
\newcommand{\zerooneloss}{\ell_{0-1}}
\newcommand{\logisticloss}{\ell_{\mathrm{log}}}
\newcommand{\maeloss}{\ell_{\mathrm{mae}}}
\newcommand{\mseloss}{\ell_{\mathrm{mse}}}
\newcommand{\one}[1]{\mathrm{1}[#1]}
\newcommand{\model}{\mathcal{M}}
\newcommand{\weight}{W}
\newcommand{\optimizer}{\mathcal{O}}
\newcommand{\sampler}{\mathcal{DS}}
\newcommand{\stopcriterion}{\mathcal{SC}}
\newcommand{\timestopcriterion}{\mathcal{SC}_{\mathrm{time}}}
\newcommand{\stepstopcriterion}{\mathcal{SC}_{\mathrm{step}}}
\newcommand{\accstopcriterion}{\mathcal{SC}_{\mathrm{acc}}}
\newcommand{\regularization}{\mathcal{R}}
\newcommand{\batch}{\mathcal{S}_{\mathrm{Batch}}}
\newcommand{\eps}{\epsilon}
\newcommand{\identity}{\mathcal{I}}
\newcommand{\tw}{\tilde \weight}
\newcommand{\basemodel}{\mathrm{Base}}
\newcommand{\rtime}{t_{\mathrm{time}}}
\newcommand{\step}{t_{\mathrm{step}}}

\algdef{SE}%
[STRUCT]%
{Struct}%
{EndStruct}%
[1]%
{\textbf{struct} \textsc{#1}}%
{\textbf{end struct}}%

\algdef{SE}[VARIABLES]{Variables}{EndVariables}
   {\algorithmicvariables}
   {\algorithmicend\ \algorithmicvariables}
\algnewcommand{\algorithmicvariables}{\textbf{state variables}}

\algdef{SE}[PARAMS]{Params}{EndParams}
   {\algorithmicparams}
   {\algorithmicend\ \algorithmicparams}
\algnewcommand{\algorithmicparams}{\textbf{parameters}}

\begin{document}

\title{Practically Solving LPN in High Noise Regimes Faster \\ Using Neural Networks}

\ifllncs
\titlerunning{Solving LPN Using Neural Networks}
\fi

\author{ 
Haozhe Jiang\thanks{Equal Contribution.}~~\thanks{IIIS, Tsinghua University. Email: \texttt{jianghz20@mails.tsinghua.edu.cn}. } 
\and Kaiyue Wen$^{*}$~\thanks{IIIS, Tsinghua University. Email: \texttt{wenky20@mails.tsinghua.edu.cn}. }
\and Yilei Chen\thanks{IIIS, Tsinghua University, Shanghai Artificial Intelligence Laboratory, and Shanghai Qi Zhi Institute. Email: \texttt{chenyilei@mail.tsinghua.edu.cn}.  
Research supported by Tsinghua University startup funding. }
}

\maketitle

\begin{abstract}
\label{sec:abstract}
We conduct a systematic study of solving the learning parity with noise problem (LPN) using neural networks. Our main contribution is designing families of two-layer neural networks that practically outperform classical algorithms in high-noise, low-dimension regimes. We consider three settings where the numbers of LPN samples are abundant, very limited, and in between. In each setting we provide neural network models that solve LPN as fast as possible. For some settings we are also able to provide theories that explain the rationale of the design of our models. 

Comparing with the previous experiments of Esser, K{\"{u}}bler, and May (CRYPTO 2017), for dimension $n=26$, noise rate $\tau = 0.498$, the ``Guess-then-Gaussian-elimination'' algorithm takes 3.12 days on 64 CPU cores, whereas our neural network algorithm takes 66 minutes on 8 GPUs. Our algorithm can also be plugged into the hybrid algorithms for solving middle or large dimension LPN instances. 
\end{abstract}

\section{Introduction}
\label{sec:intro}
Neural networks are magical, capable of learning how to play various board games \cite{alphago,CAMPBELL200257,Tesauro1995TDGammonAS} and video games \cite{Mnih2015,Vinyals2019,DBLP:journals/corr/abs-1912-06680}, how to control fusion reactors \cite{Degrave2022}, how to predict spatial structures of proteins \cite{Jumper2021}, etc. In recent years, the rise of neural networks not only revolutionizes the field of artificial intelligence but also greatly impacts other fields in computer science. It is natural to ask whether neural networks can help us with cryptography. 

One of the signature cryptographic hard problems is the Learning Parity with Noise problem (LPN), also known as decoding binary random linear codes, a canonical problem in coding theory.
Let $n$ be the dimension, $\tau\in(0, 0.5)$ be the error rate. Let $\ary{s}$ be a secret vector in $\Z_2^n$. The LPN problem asks to find the secret vector $\ary{s}$ given an oracle which, on its $i^{th}$ query, outputs a random vector $\ary{a}_i\in \Z_2^n$, and a bit $y_i := \ipd{\ary{s}}{\ary{a}_i}+e_i \pmod 2$, where $e_i$ is drawn from the error distribution that outputs 1 with probability $\tau$, 0 with probability $1-\tau$. 
When $\tau$ is very small, say $\tau\in (0,1/n)$, then as long as we obtain $\Omega(n)$ LPN samples, we can efficiently find out which $n$ of the samples are error-free and use Gaussian elimination to find the secret. However, for large $\tau\in (1/n^c, 0.5)$ where $0<c<1$, no classical or quantum algorithm is known for solving LPN in polynomial time in $n$. 
\textbf{\textbf{}}
The LPN problem was proposed by machine learning experts as a conjectured hard problem for good cryptographic use~\cite{DBLP:conf/crypto/BlumFKL93}.
Since its proposal, researchers have found numerous interesting cryptographic applications from LPN, including authentication protocols~\cite{DBLP:conf/asiacrypt/HopperB01,DBLP:conf/eurocrypt/KiltzPCJV11}, public-key encryptions~\cite{DBLP:conf/focs/Alekhnovich03,DBLP:conf/crypto/YuZ16}, identity-based encryptions~\cite{DBLP:conf/eurocrypt/BrakerskiLSV18,DBLP:conf/pkc/DottlingGHM18}, and efficient building blocks for secure multiparty computation~\cite{boyle2019efficient}. The LPN problem also inspires the formulation of the learning with errors problem  (LWE)~\cite{DBLP:journals/jacm/Regev09}, which is more powerful in building cryptographic tools.

For LPN with constant noise rates, the asymptotically fastest algorithm, due to Blum, Kalai, and Wasserman~\cite{DBLP:journals/jacm/BlumKW03}, takes $2^{O\left(\frac{n}{\log n}\right)}$ time and requires $2^{O\left(\frac{n}{\log n}\right)}$ samples.
However, for cryptosystems based on LPN, the number of samples is typically a small polynomial in $n$. In this setting, Lyubashevsky gives a $2^{O\left(\frac{n}{\log\log n}\right)}$ time algorithm~\cite{DBLP:conf/approx/Lyubashevsky05} by first amplifying the number of samples and then running the BKW algorithm. 
To improve the concrete running time for solving LPN, researchers develop more sophisticated hybrid algorithms, see for example~\cite{DBLP:conf/scn/LevieilF06,DBLP:conf/asiacrypt/GuoJL14,bogos2016solving,zhang2016faster,DBLP:conf/asiacrypt/BogosV16,DBLP:conf/crypto/EsserKM17,DBLP:conf/crypto/EsserHK0S18}. The hybrid algorithms combine the BKW reduction with other tools like the ``Guess-then-Gaussian-elimination'' (henceforth Gauss) algorithm and decoding algorithms in coding theory~\cite{DBLP:conf/asiacrypt/MayMT11,DBLP:conf/eurocrypt/BeckerJMM12}.  
Using hybrid algorithms, Esser, K{\"{u}}bler, and May~\cite{DBLP:conf/crypto/EsserKM17} show that middle-size LPN instances are within the reach of the current computation power. For example, they show that LPN with dimension $n = 135$, noise rate $\tau = 0.25$ can be solved within 5.69 days using 64 CPU cores. The practical running time for solving a larger instance of $n = 150$,  $\tau = 0.25$ is recently reported as less than 5 hours by Wiggers and Samardjiska~\cite{wiggers2021practically} by
using 80 CPU cores and more carefully designed reduction chains. %

\paragraph{Machine learning and LPN.}

The most elementary setting of machine learning is the \emph{supervised learning} paradigm. In this setting, we are given a set of labeled data that consists of some input data (e.g. pictures) and corresponding labels (e.g. whether the picture contains animals, not necessarily correct). The input data is presumed to be drawn from a fixed distribution and the given dataset often does not contain most of the possible inputs. The goal is to find a function that predicts the labels from inputs. The predictions are required to be accurate on both seen and unseen inputs. The size of the dataset required to find a good enough function is called the \emph{sample complexity} and the function is often called \emph{model}. The function is always chosen from the \emph{function class}, a pre-determined set of functions. If the function class is simple, it would be easy to find the best function, but none of them would be sophisticated enough to perform well on complex tasks. In contrast, a complex function class holds greater expressive power but may be hard to optimize and has higher sample complexity. We often call the optimization procedure \emph{learning}. Neural networks provide us with huge and adjustable expressiveness as well as an efficient learning algorithm. The recent success of neural networks largely hinges on their expressive power, as well as recent advances in big data and computing resources to make them useful.

The LPN problem perfectly fits into the supervised learning paradigm. Specifically, let the queries be inputs and the parity with noise be labels. If we can find a model that simulates the LPN oracle without noise, we can use this model to sample some data and use gaussian elimination to recover the secret. Thus it is tempting to try using the power of neural networks to break LPN. 
However, perhaps surprisingly, there are very few instances where neural networks outperform the conventional algorithms in the cryptanalysis literature. In fact, most of the successful examples are in the attack of blockciphers, see e.g.~\cite{alani2012neuro,DBLP:conf/crypto/Gohr19,benamira2021deeper}.
The only documented attempt of using neural networks in solving problems related to LPN is actually aiming at solving LWE, conducted by Wenger et al.~\cite{DBLP:journals/iacr/WengerCCL22}. But their neural network algorithm is not yet competitive with the traditional algorithms. 
Besides, the limited theoretical understanding of neural networks we have is in stark contrast to its splendid empirical achievements. Hence this paper mainly focuses on empirically demonstrating neural networks' usefulness in breaking the LPN problem with the most rudimentary networks. We hope this paper can serve as a starting point for breaking LPN and other applications in cryptography with neural networks. We also hope this paper can motivate further theoretical works on relevant problems.

\subsection{Our contributions}

Our main contribution is designing families of neural networks that practically outperform classical algorithms in high-noise, low-dimension regimes. \footnote{We release our code in \url{https://github.com/WhenWen/Solving-LPN-using-Neural-Networks.git}.} 
We consider three settings. In the \emph{abundant sample setting}, the time complexity is considered the prime. In the \emph{restricted sample setting}, efforts are made to reduce the sample complexity required by neural networks. In the \emph{moderate sample setting}, we consider optimizing time complexity given sample complexity typically given by the reduction phase of the hybrid algorithm.

\begin{table}[h]
    \centering
    \begin{tabular}{|c|c|c|c|}
    \hline 
     Model $\model$ & Initialization& Loss & Optimizer\\
    \hline
         $\basemodel_{1000}$ (\Cref{def:basemodel})& Kaiming \cite{he2015delving}& Logistic & Adam
    \\
    \hline
    \end{tabular}
    \caption{Shared Features in All Three Settings}
    \label{tab:intro_shared_hyperparameter}
\end{table}

Our main experimental results suggest that two-layer neural networks
with Adam optimizer and logistic loss function work the fastest in all three settings. Some shared features of our algorithms are presented in~\Cref{tab:intro_shared_hyperparameter}. Our neural network architectures are quite different from the ones used by Wenger et al.~\cite{DBLP:journals/iacr/WengerCCL22} where they use the transformer model to attack LWE. Let us also remark that previous attempts of using neural networks for solving decoding problems use more complicated network structures with more than five layers~\cite{nachmani2018deep,bennatan2018deep}, in contrast to our design that only uses two layers. 

Although the power of neural networks is usually hard to explain, for some settings we are able to explain the rationale of the design of our machine learning models in theory, in terms of their representation capability, optimization power, and the generalization effect. We present some first-step analysis in~\Cref{sec:theory}.

Comparing with the previous experiments of Esser, K{\"{u}}bler, and May~\cite{DBLP:conf/crypto/EsserKM17}, for dimension $n=26$, noise rate $\tau = 0.498$, the ``Guess-then-Gaussian-elimination'' algorithm takes 3.12 days on 64 CPU cores, whereas our algorithm takes $66$ minutes on 8 GPUs. For large instances like $n=125$, $\tau = 0.25$, our algorithm can also be plugged into hybrid algorithms to reduce the running time from $4.22$ days\footnote{The experiment reported in \cite[Page~28]{DBLP:conf/crypto/EsserKM17} solved an LPN instance of $n=135$, $\tau = 0.25$ in $5.69$ days. They first spent $1.47$ days to enumerate 10 bits of secrets. We didn't repeat the enumeration step and directly start from $n=125, \tau = 0.25$.} to $3.5$ hours. To reach this performance, we first use BKW to reduce $1.2e10$ LPN samples of dimension $125$, noise rate $0.25$ to $1.1e8$ LPN samples of dimension $26$, noise rate $0.498$ in $40$ minutes with $128$ CPUs. Then we enumerate the last $6$ bits and apply our neural network algorithm (\Cref{alg:moderate}) with 8 GPUs on LPN with $n=20, \tau = 0.498$ to find out $26$ bits of the secret in $66$ minutes. The whole process (BKW+neural network) is then repeated to solve another $26$ bits in another $106$ minutes. The final reduced problem has dimension $73$ and noise rate $0.25$ and can be solved by the MMT algorithm in $1$ minute. In total, the running time is $3.55$ hours.

Let us remark that we haven't compared our neural network algorithm with the more sophisticated decoding algorithms such as Walsh-Hadamard or MMT. %
But our experiments have already shown that neural networks can achieve competitive performance for solving LPN in the high-noise, low-dimension regime compared to the Gauss decoding, when running on hardware with similar costs, and there is a huge potential for improvement for our neural network algorithm.
There is also clear room for improvement in our reduction algorithm. For example, the time complexity of the second round can be significantly reduced given that we already know $26$ bits of the secret. However, our algorithm mainly focuses on accelerating the decoding phase of the hybrid algorithm. In this sense, our contributions are orthogonal to~\cite{DBLP:conf/asiacrypt/BogosV16,wiggers2021practically}, where the improvement is made possible by constructing a better reduction chain.

\paragraph{More details in the three settings.}

As mentioned, we consider three settings, named \emph{abundant}, \emph{restricted}, and \emph{moderate} sample setting. Let us now explain those three settings.

\begin{setting}[Abundant Sample] \label{setting:abundant}
In the ``Abundant Sample'' setting, we assume an unlimited amount of fresh LPN samples are given, and we look for the algorithm that solves LPN as fast as possible. This setting has a clear definition and will serve as the starting point for our algorithms.
\end{setting}

Under~\Cref{setting:abundant}, we first present the naive algorithm that directly considers the LPN problem as a supervised learning algorithm. This algorithm ignores the special structures of the LPN problem. However, our experiments in~\Cref{sec:abundant} show that given enough samples, neural networks can learn the LPN encoding perfectly. One example run on an LPN problem with dimension $20$ and noise rate $0.498$ is shown in~\Cref{fig:intro}. We also compare the time complexity of our algorithm with the Gauss algorithm in \Cref{tab:perf_intro}, showing the supremacy of our algorithm in the heavy noise regime. 

\begin{figure}
    \centering
    \includegraphics[scale = 0.4]
    {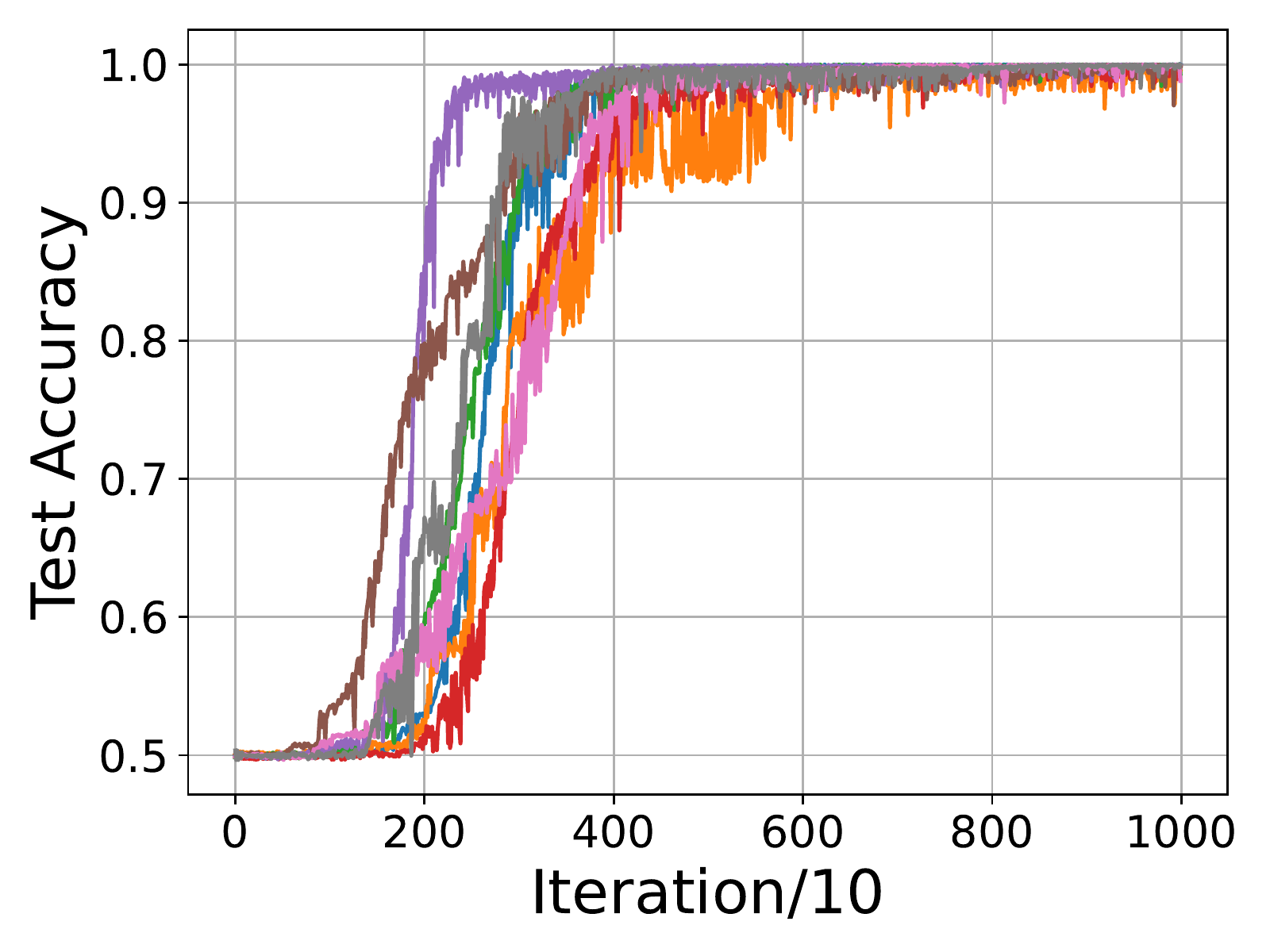}
    \caption{\textbf{Experiments on 
    $\text{LPN}_{20,\infty,0.498}$}. The horizontal axis represents the training iteration. One unit on this axis represents 10 iterations. The vertical axis represents the accuracy of the model on the clean test set without noise. There are 8 curves in this graph, each corresponding to a random initialization.}
    \label{fig:intro}
\end{figure}

\begin{table}[t]
\centering
    \begin{subtable}[t]{0.45\textwidth}
    \centering
    \begin{tabular}{cccccc}
        \toprule
        \diagbox{$n$}{$\tau$} & 0.4 & 0.45 & 0.49 & 0.495 & 0.498 \\
        \midrule
        20 & 39   & 70      & 197     & 323 & 730\\
        30 & 139  & 374     & 1576 &  & \\
        \bottomrule
    \end{tabular}
    \caption{\textbf{Neural Network}}
    \label{tab:intro_abundant}
    \end{subtable}
    \begin{subtable}[t]{0.45\textwidth}
    \centering
    \begin{tabular}{cccccc}
        \toprule
        \diagbox{$n$}{$\tau$} & 0.4 & 0.45 & 0.49 & 0.495 & 0.498\\
        \midrule
        20 & 0.40 & 5.76 & 22.0 & 312 & 6407\\
        30 & 26.4 & 682 &  & \\
        \bottomrule
    \end{tabular}
    \caption{\textbf{Gaussian Elimination}}
    \label{tab:intro_gauss}
    \end{subtable}
    \caption{\textbf{Time Complexity w.r.t Dimension and Noise Rate}. Each entry represents the running time (in seconds). The experiments for neural networks are performed on a single GPU. The experiments for Gaussian elimination are performed on a single 64 cores processor. For a neural network, the criterion for solving the LPN is that the accuracy of the network reaches $80\%$ on clean data. For Gaussian Elimination, the criterion for success is to get at least 7 correct secrets out of 10 attempts. The running time here is averaged over all runs that recover the correct secret in the time limit (3 hours). For the empty cell, no runs are successful within 3 hours.   }
    \label{tab:perf_intro}
\end{table}

\begin{setting}[Restricted Sample] \label{setting:restricted}
In the ``Restricted Sample'' setting, sample complexity is considered the prime. For the same LPN task, an algorithm that solves the task with fewer LPN oracle queries is more desirable. 
\end{setting}

In~\Cref{setting:restricted}, as the sample complexity is bounded, we would apply the \emph{search-decision} reduction first to simplify our goal. However, even distinguishing LPN data from random data is hard in this case due to the phenomenon of~\emph{overfitting}. This means when training on a small amount of data, the dataset might render the neural network to memorize all the data, resulting in poor performance on unseen inputs. We show that the commonly used regularization method named L2 regularization, or weight decay, is significantly helpful under this setting in~\Cref{sec:restricted}.
Empirical evaluations show that our algorithms achieve comparable sample complexity with state-of-the-art algorithms, the results are shown in~\Cref{tab:restricted_dimension_error}.
Together with the positive results in~\Cref{setting:abundant}, we show that neural-network based algorithms have the potential to improve the breaking algorithm of cryptography primitives.

\begin{setting}[Moderate Sample] \label{setting:moderate}
In the ``Moderate Sample'' setting, we constrain the sample complexity and seek the smallest running time. This setting is typically used as part of hybrid algorithms, where our algorithm is used to solve LPN instances reduced from BKW or other algorithms.  
\end{setting}

To further validate our points, we consider the more refined~\cref{setting:moderate}, where we consider minimizing time complexity given the sample complexity. We show with experiments that with the number of samples provided by reduction algorithms like BKW, neural networks can already learn a model with moderate accuracy, even with noise as high as $0.498$. This resolves the impracticability of the algorithm we design for~\Cref{setting:abundant}, which requires too many samples. With a combination with the traditional algorithm Gauss, we can leverage neural networks to solve low dimension, high noise LPN instances faster than previously reported~\cite{DBLP:conf/crypto/EsserKM17}.

Concluding the three settings, we show that neural networks have huge potential in solving LPN in a practical sense, for metrics spanning from pure time complexity to pure sample complexity. We also show that we can already include neural networks as a building block of the \emph{reduction-decoding} scheme to accelerate the breaking of large instances of LPN problem, which to our knowledge is the first time for neural-network based algorithm to achieve comparable, or even better, performance for LPN problem.

\paragraph{Future directions.}
At the end of the introduction we would like to mention a few interesting open problems:
\begin{enumerate}
    \item The neural network structure used by our solver is quite simple -- we only use two-layer, fully-connected networks. Are there any neural networks with more dedicated structures that help to solve LPN? 
    \item Is it possible to use our technique to solve the LWE problem? Compared to LPN, LWE uses a large modulus and uses $\ell_2$ norm to measure the length of the noise. Are those differences crucial for the competitiveness of neural networks?
    \item In addition to the previous works that use neural networks in designing decoding algorithms in coding theory~\cite{nachmani2018deep,bennatan2018deep}, our work shows the neural network has a plausible advantage in decoding binary random linear codes in the high noise regimes. 
    Can we develop practically fast neural network decoding algorithms for other codes?
\end{enumerate}

\paragraph{Organization.}
The paper is organized as follows. In~\Cref{sec:prelim} we fix some notations throughout the paper. This section also contains a brief tutorial on machine learning and neural networks. In~\Cref{sec:method} we explain our techniques and algorithms used for breaking LPN. In~\Cref{sec:experiment} we conduct experiments to demonstrate the usefulness of techniques used and compare the performance of our algorithms to SOTA prior algorithms. We also provide a guideline for tuning hyperparameters for our algorithm. In~\Cref{sec:theory} we provide theories that explain the rationales of our network architectures.

\section{Preliminary}\label{sec:prelim}

\paragraph{Notations and terminology.} 
In cryptography, the security parameter is a variable that is used to parameterize the computational complexity of the cryptographic algorithm or protocol, and the adversary's probability of breaking security. 

Let $\R, \Z, \N$ be the set of real numbers, integers, and positive integers. 
For $q\in\N_{\geq 2}$, denote $\Z/q\Z$ by $\Z_q$. 
For $n\in\N$, $[n] := \set{1, ..., n}$. A vector in $\R^n$ (represented in column form by default) is written as a bold lower-case letter, e.g. $\ary{v}$. For a vector $\ary{v}$, the $i^{th}$ component of $\ary{v}$ will be denoted by $v_i$. %
A matrix is written as a bold capital letter, e.g. $\mat{A}$. The $i^{th}$ column vector of $\mat{A}$ is denoted $\ary{a}_i$. 
The length of a vector is the $\ell_p$-norm $\|\ary{v}\|_p := (\sum v_i^p)^{1/p}$, or the infinity norm given by its largest entry $\|\ary v\|_{\infty} := \max_i\{|v_i|\}$. 
The length of a matrix is the norm of its longest column: $\|\mat{A}\|_p := \max_i \|\ary{a}_i\|_p$. 
By default, we use $\ell_2$-norm unless explicitly mentioned. 
For a binary vector $\ary{v}$, let $\HW(\ary{v})$ denote the Hamming weight of $\ary{v}$.
Let $B^n_p$ denote the open unit ball in $\R^n$ in the $\ell_p$ norm.
We will write $x\ e\ y$ as short hands for $x \times 10^y$.

When a variable $v$ is drawn uniformly random from the set $S$ we denote it as $v\la U(S)$. 
When a function $f$ is applied on a set $S$, it means $f(S) := \sum_{x\in S}f(x)$. %

\subsection{Learning Parity with Noise}\label{sec:prelim_coding}

The learning parity with noise problem (LPN) is defined as follows

\begin{definition}[LPN~\cite{DBLP:conf/crypto/BlumFKL93,DBLP:journals/jacm/BlumKW03}] \label{def:lpn}
Let $n \in\N$ be the dimension, $m\in\N$ be the number of samples, $\tau\in(0, 1/2)$ be the error rate. 
Let $\eta_\tau$ be the error distribution that output 1 with probability $\tau$, 0 with probability $1-\tau$. 
A set of $m$ LPN samples is obtained from sampling $\ary{s}\la U(\Z_2^n)$, $\mat{A}\la U(\Z_2^{n\times m})$, $\ary{e}\la\eta_\tau^m$, and outputting $(\mat{A}, \ary{y}^t := \ary{s}^t\mat{A}+\ary{e}^t \mod 2)$.

We say that an algorithm solves $\LPN_{n, m, \tau}$ if it outputs $\ary{s}$ given $\mat{A}$ and $\ary y$ with non-negligible probability. 
\end{definition}

An algorithm solves the decisional version of LPN if it distinguishes the LPN sample $\LPN_{n, m, \tau}$ from random samples over $\Z_2^{n\times m}\times \Z_2^m$ with probability greater than $1/2 + 1/poly(n)$.
The decisional LPN problem is as hard as the search version of LPN~\cite{DBLP:conf/crypto/BlumFKL93}.

The LPN problem reduces to a variant of LPN where the secret is sampled from the error distribution~\cite{DBLP:conf/crypto/ApplebaumCPS09}. The reduction is simple and important for our application so we sketch the theorem statement and the proof here. 

\begin{lemma}
\label{lem:s_sparse}
If $\LPN_{n, m, \tau}$ is hard, then so is the following variant of LPN: we sample each coordinate of the secret $\ary{s}\in\Z_q^n$ from the same distribution as the error distribution, i.e., $\eta_\tau$, and then output $m-n$ LPN samples. 
\end{lemma}
\begin{proof}
Given $m$ standard LPN samples, denoted as $(\mat{A}, \ary{y}^t := \ary{s}^t\mat{A}+\ary{e}^t \mod 2)$. 
Write $\mat{A} = [\mat{A}_1 \mid \mat{A}_2]$ where $\mat{A}_1\in\Z_2^{n\times n}$. Without a loss of generality, assume $\mat{A}_1$ is invertible (if not, pick another block of $n$ full-rank columns from $\mat{A}$ as $\mat{A}_1$). 
Write $\ary{y}^t = [\ary{y}^t_1 \mid \ary{y}^t_2]$ where $\ary{y}_1\in\Z_2^n$.  
Let $\bar{\mat{A}}:= -\mat{A}_1^{-1}\cdot \mat{A}_2$. 
Let $\bar{\ary{y}}^t := \ary{y}^t_1\cdot \bar{\mat{A}} + \ary{y}^t_2$. Then 
$\bar{\ary{y}}^t = (\ary{s}^t\mat{A}_1+\ary{e}^t_1)\cdot (-\mat{A}_1^{-1}\cdot \mat{A}_2) + (\ary{s}^t\mat{A}_2+\ary{e}^t_2) = \ary{e}^t_1\cdot \bar{\mat{A}} + \ary{e}^t_2$, meaning that
$\bar{\mat{A}}, \bar{\ary{y}}^t$ is composed of $m-n$ LPN samples where the secret is sampled from the error distribution.
\end{proof}

\subsection{Machine Learning}

\paragraph{Supervised Learning}The goal of supervised learning is to learn a function that maps inputs to labels. The input $x \in \mathcal{X}$ and the label $y  \in \mathcal{Y}$ are usually assumed to obey a fixed distribution $\distribution$ over $\mathcal{X} \times \mathcal{Y}$. Usually, $\distribution$ is not directly accessible to the learner, instead, another distribution $\datadistribution$, known as empirical distribution, is provided to the learner. This distribution is usually a uniform distribution over a finite set of inputs and labels $\trainingset \triangleq \{ (x_i, y_i)\}_{i \in [1:N]}$. This set $\trainingset$ is usually named \textit{training set} and $(x_i,y_i)$ is assumed to obey $\distribution$ independently.

The goal of \textit{learning} is to choose from a function class $\functionclass$ a function $f: X \to Y$  given $\datadistribution$. To measure the quality of $f$, \textit{loss function} $\ell: Y \times Y \to \nonneg$ is often considered. We now provide some examples of loss functions that will be used in our paper.

\begin{definition}[Zero-one Loss]
    $\zerooneloss(y_1, y_2) = \one{y_1 \neq y_2}$.
\end{definition}

\begin{definition}[Logistic Loss]
\label{def:logistic}
    $\logisticloss(y_1, y_2) = -y_2 \log( 1 - y_1) - (1 - y_2) \log y_1, y_2 \in \{0,1\},y_1 \in [0,1]$.
\end{definition}

\begin{definition}[Mean Absolute Error Loss]
    $\maeloss (y_1, y_2) = |y_1 - y_2|$.
\end{definition}

\begin{definition}[Mean Square Error Loss]
    $\mseloss (y_1, y_2) = |y_1 - y_2|^2$.
\end{definition}

\begin{definition}
    Given a loss function $\ell: Y \times Y \to \nonneg$, the \emph{population loss} $\populoss: \functionclass \to \nonneg$  is defined as
    \begin{align*}
        \populoss(f) = \E_{(x,y) \sim \distribution}[\ell(f(x), y)].
    \end{align*}
\end{definition}

For zero-one loss, $1$ minus the expected loss is also called \textit{accuracy}. We usually abuse this notation when $f$'s co-domain is $[0,1]$ by calling the accuracy of the rounding of $f$ as the accuracy of $f$. The \emph{training accuracy} is defined as accuracy with the underlying population as the uniform distribution over the training set.
The goal of learning can then be rephrased to find $f \in \functionclass$ with low population loss. Two questions then naturally arise, (1) How to evaluate population loss? and (2) How to effectively minimize population loss?

To evaluate the loss, the \textit{test set} $\testset \triangleq \{(x_i, y_i) \}_{i \in [N+1, N+M]}$ is usually considered. The element in $\testset$ is also assumed to obey $\distribution$ independently and is also independent of the elements in $\trainingset$. 
\begin{definition}
    Given a loss function $\ell: Y \times Y \to \nonneg$, the \emph{test loss} $\testloss: \functionclass \to \nonneg$  is defined as
    \begin{align*}
        \testloss(f) = \frac{1}{M}\sum_{i = N+1}^{N+M} \ell(f(x_i), y_i).
    \end{align*}
    The \emph{test accuracy} is defined as accuracy with the underlying population as the uniform distribution over the test set.
\end{definition}
When $f$ is chosen by the learning algorithm given the training data, $\testloss(f)$ can then serve as an unbiased estimator of $\populoss(f)$. In the traditional machine learning community, $\testloss$ is usually only measured once after training, and another set called \textit{validation set} is used to track the performance of the algorithm through the course of the training. However, this boundary is blurred in modern literature and we ignore this subtlety here because our final objective is to utilize machine learning to solve LPN secrets instead of fitting the data. 

To effectively minimize the loss, the learner would use a learning algorithm $\algo$ that maps training distribution to a function $f \in H$ (usually with randomness). As the learner only has access to the data distribution, $\algo$ is usually designed to minimize \emph{training loss}.
\begin{definition}
    Given a loss function $\ell: Y \times Y \to \nonneg$, the \emph{training loss} $\trainloss: \functionclass \to \nonneg$  is defined as
    \begin{align*}
        \trainloss(f) = \frac{1}{M}\sum_{i = 1}^{M} \ell(f(x_i), y_i).
    \end{align*}
\end{definition}

When trying to characterize the gap between the learned function and the best available function in the function class, the following decomposition is common in machine learning literature.
\begin{align}
&\populoss(\algo(\datadistribution)) - \underbrace{\min_{f \in \functionclass} \populoss(f)}_{\mathrm{Representation\ Gap}} \notag\\
=&\underbrace{\populoss(\algo(\datadistribution)) - \trainloss(\algo(\datadistribution))}_{\mathrm{Generalization\ Gap}} + \underbrace{\trainloss(\algo(\datadistribution)) - \trainloss(f^*) }_{\mathrm{Optimization\ Gap}} \notag \\&+ \underbrace{ \trainloss(f^*)  - \populoss(f^*)}_{\mathrm{Stochastic\ Error}}, \quad f^* = \argmin_{f \in \functionclass} \populoss(f). \label{eq:decompose}
\end{align}

The three gaps in the above equation characterize different aspects of machine learning. 
\begin{enumerate}
    \item The function class needs to be chosen to be general enough to minimize the representation gap.
    \item The learning algorithm needs to be chosen to find the best trade-off between the generalization gap and the optimization gap given the function class.
\end{enumerate}
In the recent revolution brought by neural networks, it is shown that choosing the function class as $\{ f \mid f \text{ can be represented by a fixed neural architecture}\}$ and learning algorithm as the gradient-based optimization method can have surprising effects over various domains. We will now briefly introduce neural networks and gradient-based optimization algorithms.

\paragraph{Neural Networks}

Neural networks are defined by \emph{architecture}, which maps differentiable weights to a function from $\mathcal{X}$ to $\mathcal{Y}$. This function is called the \emph{neural network} and the weights are called the \emph{parameterization} of the network. The most simple architecture is \emph{Multi-Layer Perceptron (MLP)}.
\begin{definition}[MLP]
\label{def:mlp}
    Multi-Layer Perceptron is defined as a mapping $\model$ from $\R^{d_1}$ to $\R^{d_{L+1}}$, with
    \begin{align*}
        \model[\theta_1, ..., \theta_L](x) = (\sigma_L \circ T[\theta_L] \circ  ... \circ \sigma_1 \circ T[\theta_1])(x),
    \end{align*}
    where $\theta_i = (W_i, b_i), W_i \in \R^{d_{i+1} \times d_{i}}, b_i \in \R^{d_{i+1}}$ and $T[\theta_i]$ is an affine function with $T[\theta_i](x) = W_ix + b_i$. $\sigma_i: \R^d \to \R^d $ is a function that is applied \emph{coordinate-wise} and is called \emph{activation function}. $L$ and $L - 1$ are called the number of \emph{layers} and \emph{depth} of the MLP, and $\{ d_i \}_{i = 2,..,L}$ is  called the \emph{widths} of the MLP.
\end{definition}

We now provide some examples of activation functions that will be used in our work.

\begin{definition}[ReLU]
$\relu: \R^d \to \R^d$ is defined as $(\relu(x))_i = x_i \one{x_i \ge 0} $.
\end{definition}

\begin{definition}[Sigmoid]
$\sigmoid: \R^d \to \R^d$ is defined as $(\sigmoid(x))_i = \frac{1}{1 + e^{-x_i}} $.  
\end{definition}

\begin{definition}[Cosine]
$\cosine: \R^d \to \R^d$ is defined as $(\cosine(x))_i = \cos(x_i) $.  
\end{definition}

The base model we used in this work is defined as followed.
\begin{definition}[Base Model]
\label{def:basemodel}
    Our base model is defined as MLP with depth $1$ with activation $\sigma_1 = \relu$ and $\sigma_2 = \sigmoid$. We will denote this model as $\basemodel_{d}$ with $d$ specifying $d_2$. We will write $\basemodel$ with activation $\sigma$ to indicate replacing $\sigma_1$ by $\sigma$. 
\end{definition}

\paragraph{Gradient-based Optimization}

It is common to use gradient-based optimization methods to optimize the neural network. A standard template is shown in~\Cref{alg:optim_example}. The unspecified parameters such as $\model, \weight$, and $\ell$ in the algorithm are often called \emph{hyperparameters}. 

\begin{algorithm}
\caption{Gradient-based Optimization}\label{alg:optim_example}
\begin{algorithmic}
\Require $\text{A neural network architecture }\model$
\Require $\text{An initialization parameter for the model }\weight_0$
\Require $\text{A differentiable loss function } \ell$
\Require $\text{A Stop Criterion } \stopcriterion$ 
\Require $\text{A Data Sampler } \sampler$ \Comment{See~\Cref{def:sampler}}
\Require $\text{An Optimizer } \optimizer$ \Comment{See~\Cref{def:optimizer}}
\Require $\text{A Regularization Function } \regularization$  \Comment{See~\Cref{def:regularization}}
\State $step \gets 0$
\While{$\stopcriterion$ \text{ is not reached}} \Comment{See~\Cref{def:stop_criterion}}
\State $f \gets \model[\weight_{step}]$
\State $\batch = \{(x_i, y_i)\}_{i = 1,...,B} \gets \sampler.GetData()$.
\State $\batchloss \gets \frac{1}{B} \sum_{i = 1}^B l(f(x_i), y_i) + \regularization(\weight_{step})$.  \Comment{Calculate regularized loss}
\State  $g_W \gets -\frac{\partial \batchloss}{\partial W} \mid_{W = W_{step}}$. \Comment{Calculate gradient w.r.t the model parameter}
\State $W_{step + 1} \gets \optimizer.GetUpdate(W_{step}, g_W)$.
\State $step \gets step + 1$.
\EndWhile
\State \Return $\model[\weight_{step}]$
\end{algorithmic}
\end{algorithm}

\begin{definition}[Sampler]
\label{def:sampler} 
A \emph{sampler} is a finite-state machine, on each call of method $GetData$, it will return a set of $B$ samples $\{(x_i, y_i)\}$ satisfying $x_i,y_i \sim \datadistribution$. The number $B$ is called \emph{batch size} and the $B$ samples are called a \emph{batch}.
\end{definition}

We hereby provide two examples of samplers that will be used in our papers. \Cref{alg:fix_sample} is the sampler used in Settings~\ref{setting:restricted} and~\ref{setting:moderate} and \Cref{alg:inf_sample} is the sampler used in \Cref{setting:abundant}.

\begin{algorithm}
\caption{Batch Sampler}\label{alg:fix_sample}
\begin{algorithmic}
\Params
 \State $\trainset = \{(x_i, y_i)\}_{i = 1,..,N}$, training set
\State $B \le N$, batch size
\EndParams
\Procedure{GetData}{$\ $}
\State Sample i.i.d from $\{1,2,..,N\}$ $B$ index to form index set $I_{\text{Batch}}$
\State $\batch = \{ (x_i, y_i) \mid i \in  I_{\text{Batch}}\}$.
\State \Return $\batch$.
 \EndProcedure
\end{algorithmic}
\end{algorithm}

\begin{algorithm}
\caption{Oracle Sampler}\label{alg:inf_sample}
\begin{algorithmic}
\Params
 \State $\distribution$, the underlying distribution
\State $B \le N$, batch size
\EndParams
\Procedure{GetData}{$\ $}
\State $\batch = \{ (x_{k}, y_{k}) \}_{k = 1...B}$ with $(x_k, y_k)$ i.i.d sampled from $\distribution$.
\State \Return $\batch$.
 \EndProcedure
\end{algorithmic}
\end{algorithm}

\begin{definition}[Optimizer]
\label{def:optimizer}
An \emph{optimizer} is an automaton, on each call of method $GetUpdate$, it will update states given the current parameter and gradient, and return an updated parameter.
\end{definition}

We now provide some examples of optimizers. The \emph{stochastic gradient descent (SGD) Optimizer} is shown in \Cref{alg:sgd}. When the batch size equals the size of the training set, this algorithm is often called \emph{gradient descent} directly. In our paper, we use a more complicated optimizer named \emph{Adam}, as shown in \Cref{alg:adam}. This optimizer, although poorly understood theoretically, has been widely applied across domains by the current machine-learning community.

\begin{algorithm}[h]
\caption{SGD Optimizer}\label{alg:sgd}
\begin{algorithmic}
\Params
 \State $\eta$, learning rate
 \State $\lambda$, weight decay
\EndParams
\Procedure{GetUpdate}{$W, g_W$}
\State \Return $W - \eta(\lambda W +  g_W)$.
 \EndProcedure
\end{algorithmic}
\end{algorithm}

\begin{algorithm}[t]
\caption{Adam Optimizer}\label{alg:adam}
\begin{algorithmic}
\Variables
\State $m$, first moment, initialized to be $0$.
\State $v$, second moment, initialized to be $0$.
\State $\eps$, a small positive constant, by default $1e-8$
\EndVariables
\Params
 \State $\eta$, learning rate
 \State $\lambda$, weight decay
 \State $\beta_1$, $\beta_2$, moving average factor for moments, by default $0.9, 0.999$.
\EndParams
\Procedure{GetUpdate}{$W, g_W$}
\State  $dW \gets \lambda W + g_W$.
\State  $m \gets \beta_1 m + (1 - \beta_1)dW$.
\State  $v \gets \beta_2 m + (1 - \beta_2)dW^2$.
\State  $\hat m \gets m / (1 - \beta_1)$.
\State  $\hat v \gets v / (1 - \beta_2)$.
\State  \Return $W - \eta \hat m / (\sqrt{\hat v} + \eps)$
 \EndProcedure
\end{algorithmic}
\end{algorithm}

\begin{definition}[Regularization]
\label{def:regularization}
A \emph{regularization function} is defined as a mapping from the parameter space to $\R$. 
\end{definition}

Theoretically and empirically, a proper choice of a regularization function can improve the generalization of the learned model in previous literature.
We hereby provide two examples of regularization functions that will be used in our paper. 
One can easily notice that L2 regularization (\Cref{def:l2}) applied in the gradient-based optimization method is simply another form of weight decay.
L1 regularization (\Cref{def:l1}) applied with the linear model is known as \emph{LASSO} and can induce sparsity in the model parameters (meaning the model parameters contain more zeroes).

\begin{definition}[L2 Regularization]
\label{def:l2}
    L2 Regularization $R_2: (\R^{d_1},... ,\R^{d_k}) \to \R$ with penalty factor $\lambda$ is defined as $R_2(w_1, ..., w_d) = \frac{\lambda}{2} \sum_i \|w_i\|_2^2$.
\end{definition}

\begin{definition}[L1 Regularization]
\label{def:l1}
    L1 Regularization $R_1: (\R^{d_1},... ,\R^{d_k}) \to \R$ with penalty factor $\lambda$ is defined as $R_1(w_1, ..., w_d) = \lambda \sum_i \|w_i\|_1$.
\end{definition}

\begin{definition}[Stop Criterion]
\label{def:stop_criterion}
A \emph{stop criterion} is defined as a function that returns True or False determining whether the procedure should terminate.
\end{definition}
There are typically three kinds of stop criteria, which we list as below.

\begin{definition}[Stop-by-time]
\label{def:stop_time}
A stop-by-time criterion $\timestopcriterion(\rtime)$ returns true if physical running time exceeds threshold $t$.
\end{definition}

\begin{definition}[Stop-by-step]
\label{def:stop_step}
A stop-by-step criterion $\stepstopcriterion(\step)$ returns true if the weight update step exceeds threshold $T$.
\end{definition}

\begin{definition}[Stop-by-accuracy]
\label{def:stop_acc}
A stop-by-accuracy criterion $\accstopcriterion(\dataset, \gamma)$ returns true if the accuracy of the learned function on $\dataset$ exceeds threshold $\gamma$.
\end{definition}

\section{Methods}\label{sec:method}

In this section, we will introduce our methods under the three different settings introduced in~\Cref{sec:intro}.
This section is organized as follows.~\Cref{sec:method_unlimited} introduces our methods when the data distribution coincides with the population distribution and the time complexity is a major concern. In other words, we assume an unrestricted number of fresh LPN samples are given and our goal is to learn the secret as fast as possible.
We will show that under this setting, direct application of the gradient-based optimization (cf.~\Cref{alg:optim_example}) works well both empirically and theoretically.
~\Cref{sec:method_restricted} introduces our methods when the sample complexity is greatly limited. Under this setting, reduction to decision version LPN and proper regularization is essential.
~\Cref{sec:method_abundant} introduces our methods when the sample complexity and the time complexity need a subtle balance such that our method can be utilized as 
a building block in the classical \emph{reduction-decoding} scheme of the LPN problem. Under this setting, we will use the samples to approximate the population distribution with a bootstrapping technique.

Due to~\Cref{lem:s_sparse}, we will always assume secret $\ary{s}$ has Hamming weight $\lfloor n\tau \rfloor$. 
\footnote{This is the expected hamming weight of a secret. We fix the sparsity here mainly for the convenience of comparing and adjusting other parameters.} We will first highlight the common hyperparameters for all three settings in~\Cref{tab:shared_hyperparameter}. The first row is applied in all our implemented algorithms and the second row is used in our theoretical analysis in~\Cref{sec:theory}.

We further outline other hyperparameters specified for each algorithm in~\Cref{tab:hyperparameter}. Readers should note that the hyperparameters are tuned specifically towards the typical problem specification we list out in~\Cref{tab:problem} on which our hyperparameters selection is conducted on. The power of our method is not limited to the typical specification listed and we provide generic meta~\cref{alg:abundant_meta,alg:restricted_meta} for hyperparameter selection in other cases.

\begin{table}[h]
    \centering
    \begin{tabular}{|c|c|c|c|c|c|}
    \hline 
    & Model $\model$ & Initialization & Regularization & Loss & Optimizer\\
    \hline
        Practical  &  $\basemodel_{1000}$ (\Cref{def:basemodel})& Kaiming & None & Logistic & Adam\\
    \hline
       Theory  &  $\basemodel_{d}$ with smooth activation $\sigma$ & Any & None & MAE & SGD
    \\
    \hline
    \end{tabular}
    \caption{Shared Hyperparameters}
    \label{tab:shared_hyperparameter}
\end{table}

\begin{table}[]
    \centering
    \begin{tabular}{|c|c|c|c|c|c|} 
    \hline & Sampler & Learning Rate & Weight Decay &  Batch Size & Stop Criterion \\
    \hline
      \Cref{alg:abundant}   & Oracle & $2e-4$ to $6e-3$ & 0 & 131072 to 1048576& $\timestopcriterion(\rtime)$ \\
    \hline
      \Cref{alg:abundant_theory} & Oracle & Any
    & Any & Any &$\stepstopcriterion(\step)$\\
    \hline
    \Cref{alg:restricted} & Fix Batch & $2e-5$ to $1e-4$ & $2e-3$  & Training Set Size&$\stepstopcriterion(\step)$ \\
    \hline
    \Cref{alg:moderate} & Fix Batch & $2e-3$ & 0 & 1048576 & $\timestopcriterion(\rtime)$ \\
    \hline
    \end{tabular}
    \caption{Hyperparameters for Different Algorithms}
    \label{tab:hyperparameter}
\end{table}
\

\begin{table}[]
    \centering
    \begin{tabular}{|c|c|c|c|} 
    \hline & Dimension $n$ & Noise Rate $\tau$ & Sample Size\\
    \hline
      \Cref{alg:abundant}   & 20 to 40 & 0.4 to 0.498 & Abundant \\
    \hline
    \Cref{alg:restricted} & 20 to 50 & 0.2 to 0.3  & Less than $1e6$ \\
    \hline
    \Cref{alg:moderate} & 20  & 0.498 & 8e7\\
    \hline
    \end{tabular}
    \caption{Typical Problem Specification for Different Algorithms }
    \label{tab:problem}
\end{table}

\subsection{Abundant Sample}\label{sec:method_unlimited}

In~\Cref{setting:abundant}, we assume the learner has access to the Oracle Sampler (\Cref{alg:inf_sample}) with underlying distribution following LPN instances with a fixed secret, a fixed noise rate $\tau$, and any batch size $B$. Under this setting, the memory of the devices is the key constraint and the goal of this algorithm is to 
reduce the time complexity. We propose to use a direct variant of the gradient-based optimization algorithm~\Cref{alg:optim_example}, which is shown in~\Cref{alg:abundant}.

\begin{algorithm}
    \caption{Abundant Sample Algorithm: Practical Version}\label{alg:abundant}
    \begin{algorithmic}
    \Require $n$, the dimension, $\tau$ the error rate.
    \State Run~\Cref{alg:optim_example} with hyperparameters specified in~\Cref{tab:hyperparameter,tab:shared_hyperparameter} to get a learned model $\model[\weight_T]$.
    \State Set $\hat s$ as all zero vector in $\R^n$.
    \For{$i \in \{1,2,...n\}$}
    \State Set $\hat s[i] = \one{\model[\weight_T][e_i] > 0.5}$ with $e_i \in \R^n$ as the unit vector with the $i-$th coordinate being $1$.
    \EndFor
    \State \Return $\hat s$.
    \end{algorithmic}
\end{algorithm}

One may notice that there are some unspecified hyperparameters, including the time threshold $\rtime$ in~\Cref{alg:abundant}. Also for settings outside the scope of~\Cref{tab:problem}, some currently fixed parameters such as the learning rate $\eta$ and the batch size $B$ may also require tuning based on the problem setting, i.e, $n$ and $\tau$. Under such case, we propose to use~\Cref{alg:abundant_meta}, a meta algorithm for hyperparameter selection. In~\Cref{sec:abundant_hyper}, we show that~\Cref{alg:abundant_meta} 
effectively finds hyperparameters that leads to low time complexity of~\Cref{alg:abundant}, for example,~\Cref{alg:abundant} with $\eta =6e-3$ and $B = 1048576$, can solve $n = 20, \tau = 0.495$ in $6$ minutes with a single GPU.
An example running results of~\Cref{alg:abundant_meta} is shown in~\Cref{tab:ablation_1_lr_bz}.

\begin{algorithm}
    \caption{Abundant Sample Algorithm: Theoretical Version}\label{alg:abundant_theory}
    \begin{algorithmic}
    \Require $n$, the dimension, $\tau$ the error rate.
    \State Run~\Cref{alg:optim_example} with hyperparameters specified in~\Cref{tab:hyperparameter,tab:shared_hyperparameter} to get a learned model $\model[\weight_T]$ and return the model.
    \end{algorithmic}
\end{algorithm}
\begin{algorithm}
    \caption{Abundant Sample Meta Algorithm For Hyperparameters Selection}\label{alg:abundant_meta}
    \begin{algorithmic}
    \Require $n$, the dimension, $\tau$ the error rate.
    \Require $\gamma$, the accuracy threshold, fixed as $80\%$ in our experiments,.
    \Require $Repeat$,  the number of different datasets to estimate $\rtime$, fixed as $3$ in our experiments
    \Require A set of hyperparameters profiles $\mathcal{P}$.
    \State Initialize an empty hashmap $\mathrm{Map}_t$ to record running time.
    \For{Profile $P \in \mathcal{P}$}
    \State Set $\mathrm{Map}_{\rtime}[P] = \infty$.
    \For{$\mathrm{r} \le Repeat$}
    \State Randomly generate secret $s$ with Hamming weight $\lfloor n\tau \rfloor$.
    \State Create an oracle sampler $\sampler$ with batch size $B$ and error rate $\tau$ corresponding to $s$.
    \State Sample a test dataset $\testset$ using $\sampler$, with size $O(n)$. 
    \State \Comment{Fixed as 131072 in our experiments.}
    \State Set stop criterion $\stopcriterion'$ as $\accstopcriterion(\testset, \gamma)$.
    \State Run a variant of~\Cref{alg:abundant} with stop criterion $\stopcriterion'$ and  profile $P$ and record running time $\rtime$.
    \State Set $\mathrm{Map}_{\rtime}[(B, \eta)] = \min\{\rtime, \mathrm{Map}_{\rtime}[(B, \eta)] \}$.
    \State \Comment{Only require solving by constant probability.}
    \EndFor
    \EndFor
    \State \Return $(P,\mathrm{Map}_{\rtime}[P])$ with the smallest $\mathrm{Map}_{\rtime}[P]$ for $P \in \mathcal{P}$. 
    \end{algorithmic}
\end{algorithm}

We also perform some theoretical analysis on the time complexity of our method, showing that the dependency on $\tau$ is merely $(\frac{1}{2} - \tau)^{-2}$. 
However, due to the complexity of our choice of the optimizer and the loss function in~\Cref{alg:abundant}, our analysis is conducted on a simpler version (\Cref{alg:abundant_theory}) with more clear theoretical structure.
The analysis is deferred to~\Cref{sec:theory_optimization}.

\subsection{Restricted Sample}\label{sec:method_restricted}

In~\Cref{setting:restricted}, the sample complexity is assumed to be highly limited. 
Due to the restriction on sample complexity, we could no longer expect the learned model to fully recover the secret directly as in~\Cref{sec:method_unlimited}.
Hence we design our algorithm to solve the decisional version of LPN and use neural network as a reduction method instead. We present our algorithm in~\Cref{alg:restricted}.
The returning value of~\Cref{alg:restricted} is with high probability the value of secret $s$ at index $n-1$.
In practice, the inner for loop of~\Cref{alg:restricted} is parallelized across multiple GPUs.

\begin{algorithm}
    \caption{Restricted Sample Algorithm}\label{alg:restricted}
    \begin{algorithmic}
    \Require $\dataset$ with $m$ samples drawn from LPN with dimension $n$ and error rate $\tau$.
    \Require $repeat$ the number of random initializations to try, typically set as $8$.
    \Require $\gamma$, the accuracy threshold.
    \State Split $\dataset$ into $\trainset$ and $\testset$, each with $\lfloor m/2 \rfloor$ sample.
    \For{Guess $g \in \{0,1\}$}
    \State Generate $\trainset^g = \{ (x[1:d-1], y + x[d] \times g \mod 2) \mid (x,y) \in \trainset \}$.
    \State Generate $\testset^g = \{ (x[1:d-1], y + x[d] \times g \mod 2) \mid (x,y) \in \testset \}$.
    \State Set sampler $\sampler$ as batch sampler with batch size $\lfloor m/2 \rfloor$ on $\trainset^g$.
    \For{$r \le repeat $}
    \State Run~\Cref{alg:optim_example} with the hyperparameters specified in~\Cref{tab:hyperparameter,tab:shared_hyperparameter} to get a learned model $\model[\weight_T]$.
    \If{the accuracy of $\model[\weight_T]$ on $\testset^g$ exceeds $\gamma$} {\Return $g$.}
    \EndIf
    \EndFor
    \EndFor
    \end{algorithmic}
\end{algorithm}

Again, there are some unspecified hyperparameters in~\Cref{alg:restricted} that need to be pivoted down by experiments.
In our experiments, we mainly consider $n \le 50$, $\tau \le 0.3$ and $m \in [2^{10}, 2^{20}]$. In this regime, we propose to set $\step = 300k$ and $\gamma$ as $1/2 + \sqrt{\log(20)/{m}}$. 
The dependency on $\tau$ is dropped here because the noise we considered here is relatively small. We observe in practice that the accuracy of test set on successful runs almost always exceeds this threshold by a large margin. For other hyperparameters including the exact value of sample size $m$, we apply meta algorithm~\Cref{alg:restricted_meta}.

We would like to stress a few major differences in hyperparameters selection between~\Cref{setting:abundant} and~\Cref{setting:restricted}.
\begin{enumerate}
    \item L2 regularization, or weight decay, is not helpful in~\Cref{setting:abundant}, but can reduce the sample complexity significantly in~\Cref{setting:restricted}.
    \item Under~\Cref{setting:restricted}, it is generally better to use the whole dataset as a batch instead of using a smaller batch size. 
    \item The learning rate required by~\Cref{alg:restricted} is typically smaller than the learning rate required by~\Cref{alg:abundant} by a factor of $10$ to $100$.
\end{enumerate}

\begin{algorithm}[t]
    \caption{Restricted Sample Meta Algorithm For Hyperparameters Selection}\label{alg:restricted_meta}
    \begin{algorithmic}
    \Require $n$, the dimension, $\tau$ the error rate.
    \Require A set of sample numbers $m_0 < m_1 < m_2 < ... < m_L$.
    \Require $Repeat$, the number of different datasets we tested on, fixed as $3$ in our experiments
    \Require A set of hyperparameters profiles $\mathcal{P}$.
    \For{Profile  $P \in \mathcal{P}$}
    \State Set $\mathrm{Map}_m[P] = 0$ to record the sample needed to perform reduction.
    \State Set low and high index to perform a binary search, $Left = 0$ and $Right = L$.
    \While{$Left \neq Right$}
    \State Middle index $Mid = \lfloor (Left + Right)/2 \rfloor$.
    \State Set sample number $m = m_{Mid}$.
    \State Set success Count $c = 0$.
    \For{$\mathrm{r} \le Repeat$}
    \State Randomly generate secret $s$ with Hamming weight $\lfloor n\tau \rfloor$ 
    \State Sample a dataset $\dataset$ generated with secret $s$ and error rate $\tau$, with size $m$. 
    \State Run a variant of~\Cref{alg:restricted} with hyperparameter profile $P$ 
    \State Increment success count $c$ by $1$ if the return value equals $s[d-1]$.
    \EndFor
    \If{$c \ge \lfloor 2Repeat / 3 \rfloor$} {Right = Mid}
    \EndIf
    \If{$c < \lfloor 2Repeat / 3 \rfloor$} {Left = Mid + 1}
    \EndIf
    \EndWhile
    \State Set $\mathrm{Map}_m[P] = m_{Left}$.
    \EndFor
    \State \Return $(P, \mathrm{Map}_m[P])$ with the smallest $\mathrm{Map}_m[P]$.
    \end{algorithmic}
\end{algorithm}

\subsection{Moderate Sample}\label{sec:method_abundant}
\begin{algorithm}[h]
    \caption{Moderate Sample Algorithm}\label{alg:moderate}
    \begin{algorithmic}
    \Require $\dataset$ with $m$ samples drawn from LPN with dimension $n$ and error rate $\tau$.
    \Require Repeat number $repeat$ for starting with different initializations, set to $1$ in our experiments.
    \Require Repeat number $repeat_{post}$ for post processing run, set to $20$ in our experiments.
    \Require The size of boosting set $m'$ for post processing run.
    \Require The hypothesis test error rate threshold $\tau'$.
    \State Set $m_1 = \frac{2n}{(1/2 - \tau)^2}$. \Comment{See Lemma 3 in \cite{DBLP:conf/crypto/EsserKM17}}
    \State Split the dataset with a training set $\trainingset$ of size $m - m_1$ and a test set $\testset$ of size $m_1$.
    \For{$r \le repeat$}
    \State Run~\Cref{alg:optim_example} with the parameters specified in~\Cref{tab:hyperparameter,tab:shared_hyperparameter} to get a learned model $\model[\weight_T]$.
    \State Randomly generate $m'$ boolean vector $x_i \in \{0,1\}^{n + 1}$ and use $\model[\weight_T]$ to predict the pseudo label $\one{\model[\weight_T](x_i[1:n]) > 0.5} + x_i[n+1] \mod 2$ to form $\boostingset$. 
    \State \Comment{Rebalance Step}
    \State Run Pooled Gaussian algorithm \cite{DBLP:conf/crypto/EsserKM17} on $\boostingset$ for $repeat_{post}$ number of times to get a set of possible secret (discard the last bit) using hypothesis test error rate threshold $\tau'$. 
    \State \Comment{Post Processing Step}
    \State Return the secret if one of them reaches accuracy $1 - \tau - \sqrt{\frac{3(\frac{1}{2} - \tau) n}{m_1}}$  on $\testset$.
    \EndFor
    \end{algorithmic}
\end{algorithm}

\Cref{setting:restricted} and~\Cref{setting:abundant} can be viewed as two extreme settings where in the first one only time complexity is considered and in the second one only sample complexity is considered. Although they both show interesting properties and strong performances, difficulties are faced 
when trying to fit them in the classical \emph{reduction-decoding} paradigm of LPN algorithms, in which the sample complexity after reduction is mediocre. This calls for investigating whether and how we can apply machine learning methods as a part of the \emph{reduction-decoding} paradigm, especially as the decoding algorithm.

Our work gives an affirmative answer to the first question and proposes the following~\Cref{alg:moderate} as a candidate. 
We would like to point out that the algorithm design process of~\Cref{alg:moderate} centered around LPN problems with medium dimension $n \approx 20$ and heavy error rate $\tau > 0.495$, as this is the 
setting after the reduction phase where classical algorithms like Pooled Gaussian and MMT require huge time complexity.
Similar to the previous section, a meta algorithm is required to determine the running time $\rtime$ and predicted accuracy $\tau'$. The algorithm is analogous to~\Cref{alg:abundant_meta} and is omitted here.

One can see the major differences between~\Cref{alg:abundant} and~\Cref{alg:moderate} is that
\begin{enumerate}
    \item The sampler is replaced by the fixed batch sampler. We use this bootstrapping technique to simulate sampling from $\distribution$ with the sample we possess.
    \item The direct inference is replaced by a pooled Gaussian step. This is because the training part of~\Cref{alg:moderate} typically returns a model with low accuracy, for example around $52\%$ on clean data when $n = 20, \tau = 0.498$ and $m - m_1  = 1e8$. We empirically observe that the rebalance and post-processing step usually takes less than 2 minutes to recover the correct secret under the typical setting. The experiment details are deferred to~\Cref{sec:experiment}.
\end{enumerate}

To simulate the training dynamics of~\Cref{alg:moderate} with~\Cref{alg:abundant}, we propose to use the same hyperparameter setting except for the sampler. The initial sample $m$ under this setting is mostly determined by the reduction phase. For example, in our case study in~\Cref{sec:experiment}, it is set to $1e8$ as the BKW algorithm reduces LPN problem with $n = 125, \tau = 0.2, m = 1.2e10$ to LPN problem with 
$n = 26, p = 0.498, m = 1.1e8$ in about $0.5$ hours with 128 cores server.

\section{Experiment}
\label{sec:experiment}

In this section, we perform experiments on our methods and classical methods like BKW and Gauss as well. Without otherwise mention, the experiment results we report are conducted on 8 NVIDIA 3090 GPUs, which cost approximately 10000 dollars. For other experiments using CPUs, we use a 128 CPU cores server with 496GB memory. Specifically, the server contains two AMD EPYC 7742 64-Core Processors, which cost approximately 10000 dollars as well.

This section is organized as follows. The results for three settings are shown respectively in~\Cref{sec:abundant,sec:restricted,sec:moderate}. In each subsection, we will first show our method's overall performance with a comparison to classical algorithms. Then we will present a case study section, in which pilot experiments are demonstrated to show the experiment observation we have under the setting. It will be followed by a hyperparameter selection section, in which we showcase how our meta hyperparameter selection algorithm is performed and also our conclusion on the importance of different hyperparameters.

\subsection{Abundant Sample Setting}\label{sec:abundant}

In this section, we study the empirical performance of Algorithm \ref{alg:abundant}. The performance of Algorithm \ref{alg:abundant} is summarized in~\Cref{tab:abundant}. As a comparison, we list the performance of Gaussian Elimination in~\Cref{tab:gauss}. The hyperparameters may not be optimal for all $n$ and $\tau$. Nevertheless, it loosely reflects how the performance of Algorithm \ref{alg:abundant} varies with $n$ and $\tau$. A key observation is that the performance does not deteriorate much when we raise the noise. For instance, the runtime only triples when raising $\tau$ from 0.45 to 0.49. However, the running time rises quickly when we raise the dimension. In contrast, the running time of Gaussian Elimination rockets with both dimension and noise rate. These features makes~\Cref{alg:abundant} competitive in medium dimension with super high noise cases (note that LPN instances with medium dimension and super high noise often appear at the final stage of the BKW reduction or other dimension-reduction algorithms). In~\Cref{sec:abundant_case}, we will first illustrate typical phenomena that arise when running the algorithm. In~\Cref{sec:abundant_hyper}, we show how to tune the algorithm to its best performance.

\begin{table}[t]
\centering
    \begin{subtable}[t]{0.45\textwidth}
    \centering
    \begin{tabular}{cccccc}
        \toprule
        \diagbox{$n$}{$\tau$} & 0.4 & 0.45 & 0.49 & 0.495 & 0.498 \\
        \midrule
        20 & 39   & 70      & 197     & 323 & 730\\
        30 & 139  & 374     & 1576 &  & \\
        \bottomrule
    \end{tabular}
    \caption{\textbf{Neural Network}}
    \label{tab:abundant}
    \end{subtable}
    \begin{subtable}[t]{0.45\textwidth}
    \centering
    \begin{tabular}{cccccc}
        \toprule
        \diagbox{$n$}{$\tau$} & 0.4 & 0.45 & 0.49 & 0.495 & 0.498\\
        \midrule
        20 & 0.40 & 5.76 & 22.0 & 312 & 6407\\
        30 & 26.4 & 682 &  & \\
        \bottomrule
    \end{tabular}
    \caption{\textbf{Gaussian Elimination}}
    \label{tab:gauss}
    \end{subtable}
    \caption{\textbf{Time Complexity w.r.t. Dimension and Noise Rate}. Each entry represents the running time (in seconds) for the corresponding algorithm to solve the corresponding LPN instance with constant probability. For a neural network, the criterion for solving the LPN is that the accuracy of the network reaches $80\%$ on clean data. For Gaussian Elimination, the criterion for success is to get at least 7 correct secrets out of 10 attempts. The experiments presented in~\Cref{tab:abundant} are performed on a single GPU and in~\Cref{tab:gauss} on a single 64 cores Processor. The running time here is averaged over all runs that recover the correct secret in the time limit (3 hours). For the empty cell, no runs are successful within 3 hours. }
    \label{tab:perf}
\end{table}

\subsubsection{Case Study} 
\label{sec:abundant_case}
We run Algorithm \ref{alg:abundant} on $\LPN_{44,\infty,0.2}$. The training accuracy and the test accuracy are shown in Figure \ref{fig:abundant_case}. 

\begin{figure}
    \centering
    \begin{subfigure}{0.45\textwidth}
    \centering
        \includegraphics[scale = 0.35]{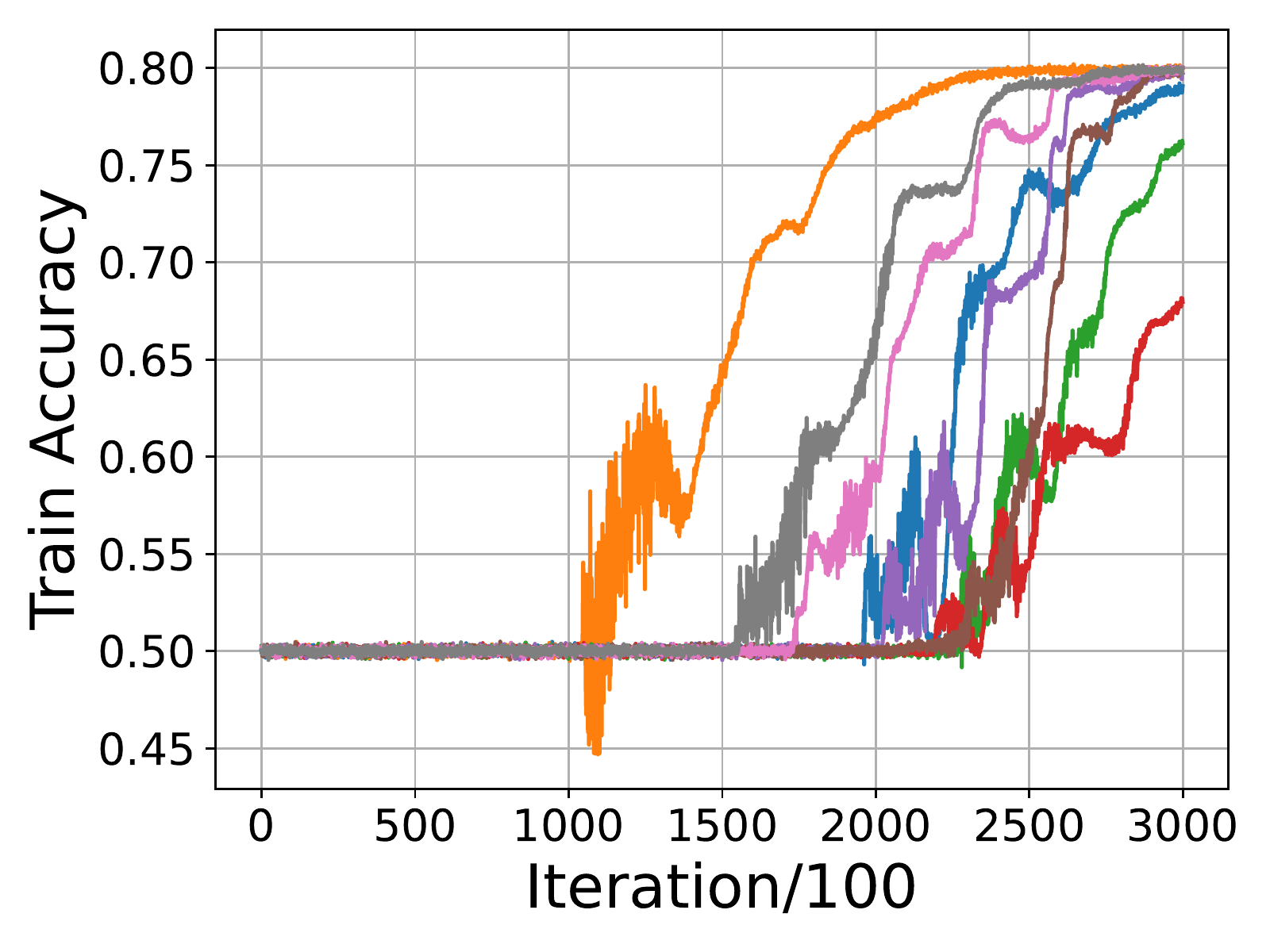}
        \caption{Training Accuracy}
    \end{subfigure}
    \begin{subfigure}{0.45\textwidth}
    \centering
        \includegraphics[scale = 0.35]{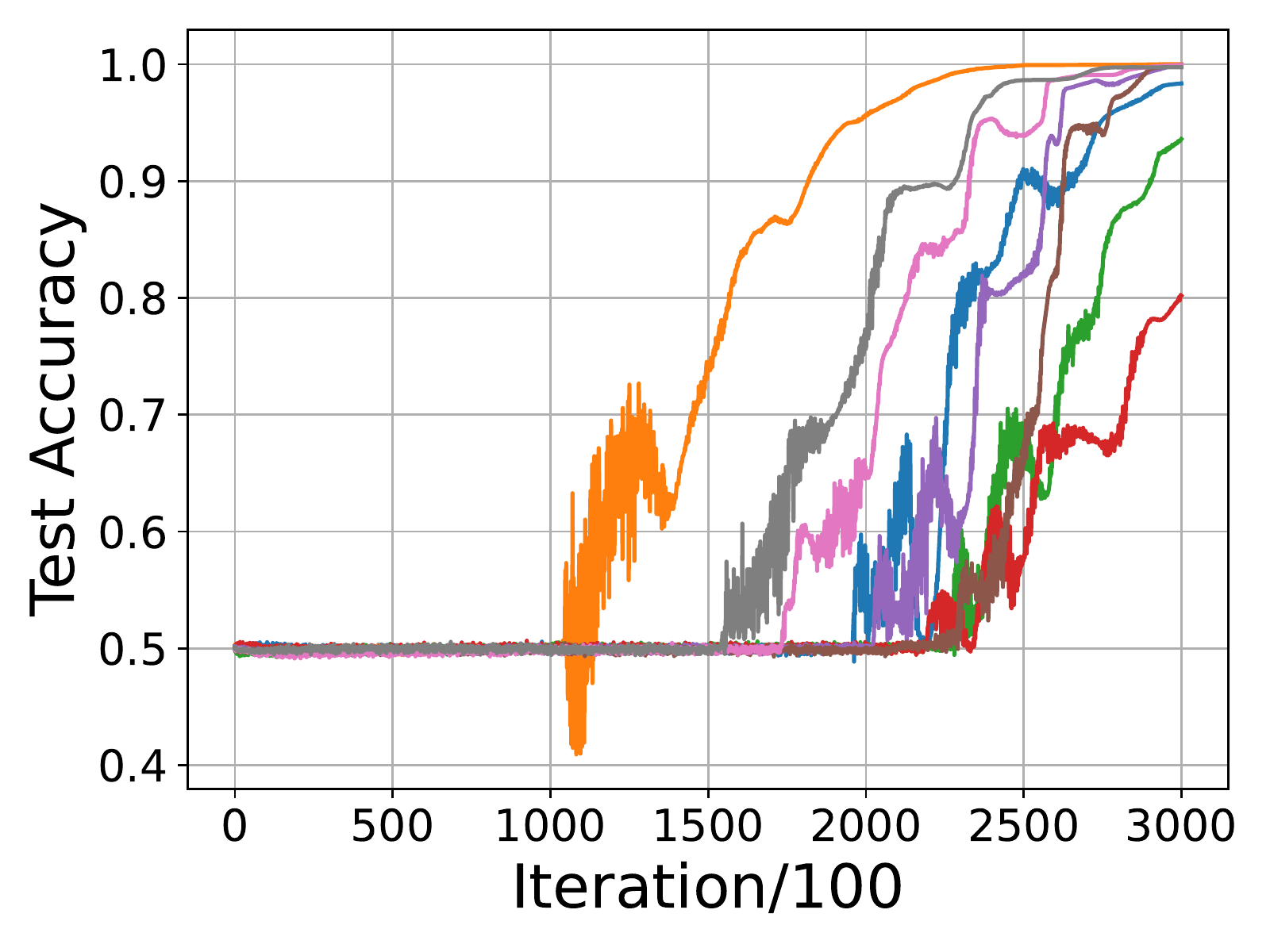}
        \caption{Test Accuracy}
    \end{subfigure}
    \caption{\textbf{Experiments on $\text{LPN}_{44,\infty,0.2}$}. Figures (a) and (b) show the training accuracy and the evaluation accuracy respectively. The horizontal axis represents the training iteration. One unit on this axis represents 100 iterations. The vertical axis represents the accuracy of the model on the corresponding dataset. The training accuracy here is calculated on the batch used for training the network before the gradient is used to update the weight. There are 8 curves in every graph, each corresponding to a random initialization. The same type of graphs will appear frequently in this paper.}
    \label{fig:abundant_case}
\end{figure}

\begin{Empirical}
    Under~\Cref{setting:abundant}, we find that (i) The randomness in initialization has little impact on the running time or the final converged accuracy of the model. (ii) The training and test accuracy tie closely and test accuracy usually reaches $100\%$ eventually on clean data after it departs from $50\%$.
\end{Empirical}

\paragraph{Initialization} We use Kaiming initializations. As depicted in Figure \ref{fig:abundant_case}, it is observed that whether a model can learn the task is independent of the randomness in the Kaiming initialization. We also note that the running time needed to learn the LPN instance does not differ much among different initializations. Considering this phenomenon, we do not try multiple initializations in designing~\Cref{alg:abundant}.

\paragraph{Matching Training and Test Accuracy.} One can observe an almost identical value between the training and test accuracy in~\Cref{fig:abundant_case}.  This is expected because, on each train iteration, fresh data is sampled so the training accuracy, similar to the test accuracy, reflects the population accuracy. 
It is further observed that after the test accuracy starts to depart from 50\%, it always reaches 100\% accuracy eventually (here 100\% refers to the test accuracy on noiseless data). This phenomenon suggests that most local minima of $\populoss(f)$ found by algorithm \ref{alg:optim_example} are global minima of $\populoss(f)$.

\paragraph{Direct Inference of the Secret} We experiment with models with accuracy $80\%$ on the post processing step of~\Cref{alg:abundant} and find out that the post processing step returns secret with high probability with typical time complexity less than $1$s.

\subsubsection{Hyperparameter Selection}
\label{sec:abundant_hyper}

In this section, we fix an LPN problem setup with $n = 20$ and $\tau = 0.498$ and showcase the application of our~\Cref{alg:abundant_meta}. Recalling our goal is to find hyperparameters that minimize the time complexity, we report the minimal running time in seconds for the model to reach accuracy $80\%$ on noiseless data for one out of the four datasets we experimented on. For each experiment, we fix all except the investigated hyperparameters the same as in the corresponding column of~\Cref{alg:abundant} in~\Cref{tab:hyperparameter,tab:shared_hyperparameter}.

\begin{Empirical}
Under~\Cref{setting:abundant},  we find that (i) Larger batch size has tolerance for larger learning rate, however, all sufficiently large batch size yields similar running time. (ii) The depth of the network should be fixed to 1 and the width should be carefully tuned.
\end{Empirical}

\paragraph{Learning Rate and Batch Size} The performance of algorithm \Cref{alg:abundant} under different learning rates and batch sizes are listed in Table \ref{tab:ablation_1_lr_bz}. For a fixed batch size, if the learning rate is too small, the model will fit slowly to prolong the learning time. If on contrary, the learning rate is too large, the gradient steps would be so large that the learner cannot locate the local minima of the loss. The upper threshold of the learning rate increases with respect to the batch size because, for a fixed learning rate, a small batch size would result in an inaccurate gradient estimate and render the learning to fail. On the contrary, a large batch size can make the running time unnecessarily long. The best running times for all batch sizes are similar and for all learning rates are similar as well. As a result, we would recommend one first finds a batch size and learning rate combination that can fit the data and then tune one of the factors while fixing the other for higher performance.

\begin{table}[h]
    \centering
    \begin{tabular}{ccccc}
        \toprule
        \diagbox{Learning Rate}{$\log_2$Batch Size} & 17 & 18 & 19 & 20 \\
        \midrule
        $6\times10^{-5}$ & 1807    & 2405    & 3569    & 4960\\
        $2\times10^{-4}$ & 806     & 1140    & 1389    & 1650\\
        $6\times10^{-4}$ & $>4000$ & $>4000$ & 692     & 970\\
        $2\times10^{-3}$ & $>4000$ & $>4000$ & $>4000$ & 750\\
        $6\times10^{-3}$ & $>4000$ & $>4000$ & $>4000$ & 723\\
        \bottomrule
    \end{tabular}
    \caption{\textbf{Time Complexity w.r.t. Learning Rate and Batch Size}. Each entry represents the running time (in seconds) for~\Cref{alg:abundant} with corresponding hyperparameters to solve $\text{LPN}_{20,\infty,0.498}$ with probability approximately $1/4$.}
    \label{tab:ablation_1_lr_bz}
\end{table}

\paragraph{L2 Regularization} L2 regularization does not help to learn at all and increases running time over all settings we experimented on. As mentioned in~\Cref{sec:abundant_case}, generalization is not an issue under this setting hence L2 regularization is unnecessary.

\paragraph{Width and Depth of the Model}
We experiment with different architectures including MLP models with different widths in $\{500, 1000, 2000\}$ and depths in $\{1,2,3\}$. The base model we apply outperforms other architectures significantly and returns the correct secret at least $5$ times faster. This experiment, alongside the following architecture experiment in~\Cref{sec:restricted_hyper}, shows that one should tend to use a shallow neural network with depth $1$ and carefully tune the width of the model.

\subsection{Restricted Sample Setting}\label{sec:restricted}

In this section, we study empirically the minimal sample complexity required by neural networks to learn LPN problems. The investigation here is mostly limited to the case where the noise is low, as a complement to the cases in Settings~\ref{setting:abundant} and~\ref{setting:moderate}. As mentioned in~\Cref{sec:method_restricted}, we utilize~\Cref{alg:restricted}, which aims to solve the decision version of the problem. The performance of~\Cref{alg:restricted} is shown in~\Cref{tab:restricted_dimension_error}. These sample complexities are typically comparable with the classical algorithm, such as BKW. In~\Cref{sec:restricted_case}, we will show some important observations that guide us in designing~\Cref{alg:restricted}. In~\Cref{sec:restricted_hyper}, we will show how to use~\Cref{alg:restricted_meta} to find hyperparameters of the~\cref{alg:restricted}.
\begin{table}[h]
    \centering
    \begin{tabular}{cccc}
        \toprule
        \diagbox{Dimension}{Error Rate} & 0.1  & 0.2 & 0.3\\
        \midrule
        25 & 7    & 10.5 & 13  \\
        30 & 9.5  & 12.5 & 16  \\
        35 & 9.5  & 15 & 17.5  \\
        40 & 10.5 & 16 & 20.5  \\
        \bottomrule
    \end{tabular}
    \caption{\textbf{Sample Complexity w.r.t. Dimension and Noise Rate}. Each entry represents the logarithm of the minimal number of samples with base $2$ for~\Cref{alg:restricted} to return the correct guess of the last bit of the secret with a probability of approximately $2/3$.}\label{tab:restricted_dimension_error}
\end{table}

\subsubsection{Case Study}
\label{sec:restricted_case}
\begin{figure}[t]
    \centering
    \begin{subfigure}{0.3\textwidth}
        \centering
        \includegraphics[scale = 0.23]{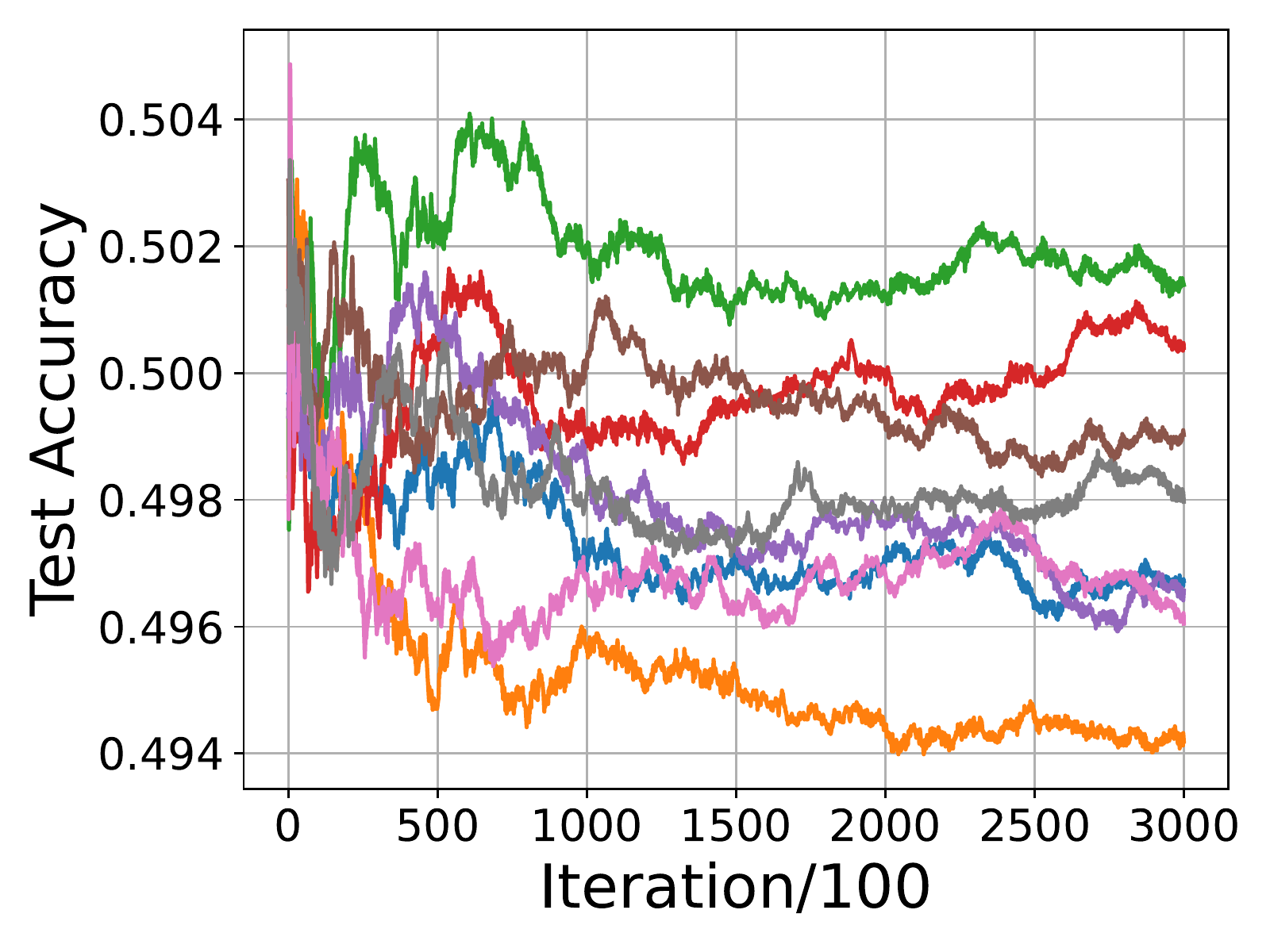}
        \caption{}
    \end{subfigure}
    \begin{subfigure}{0.3\textwidth}
        \centering
        \includegraphics[scale = 0.23]{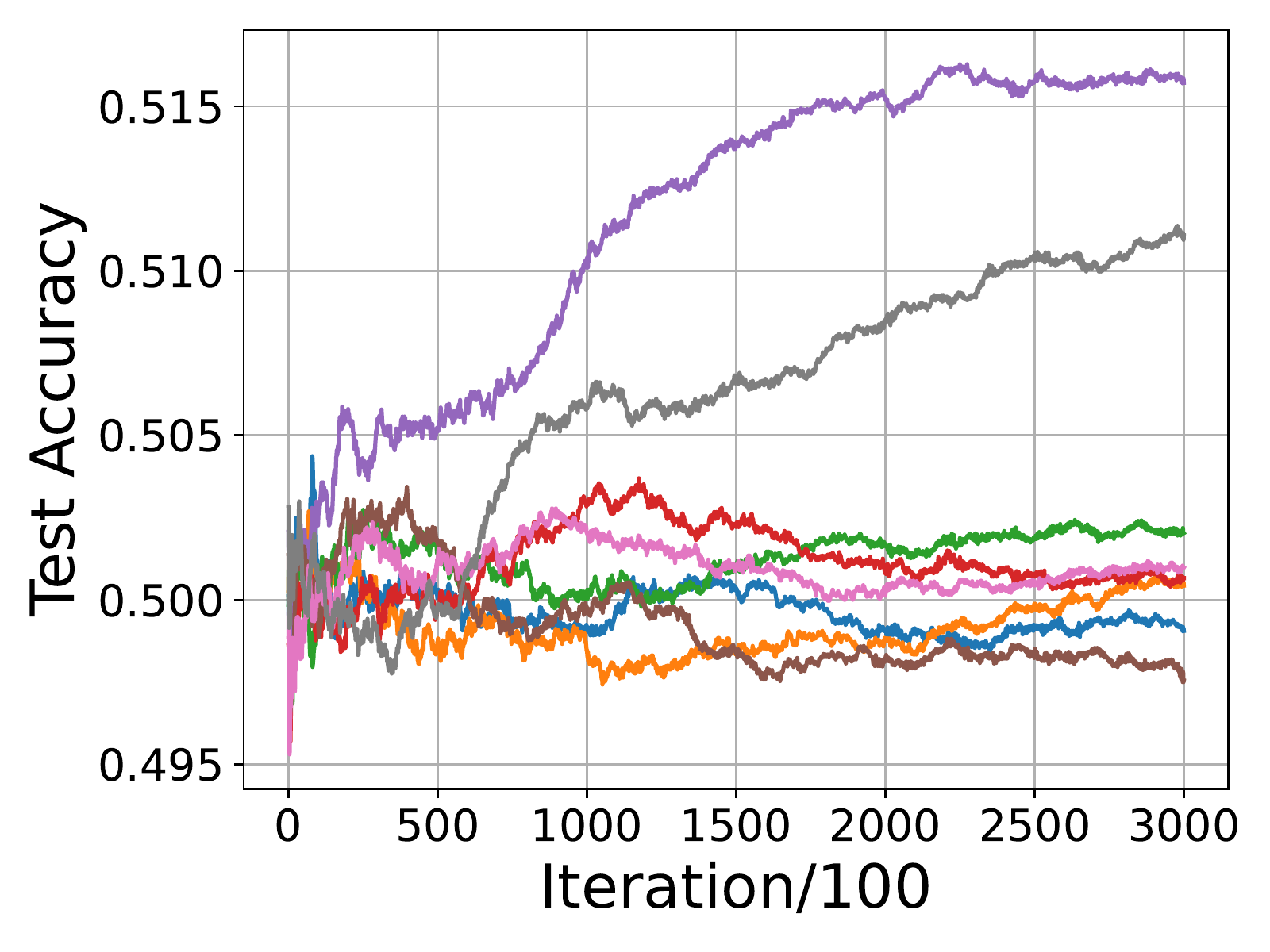}
        \caption{}
    \end{subfigure}
    \begin{subfigure}{0.3\textwidth}
        \centering
        \includegraphics[scale = 0.23]{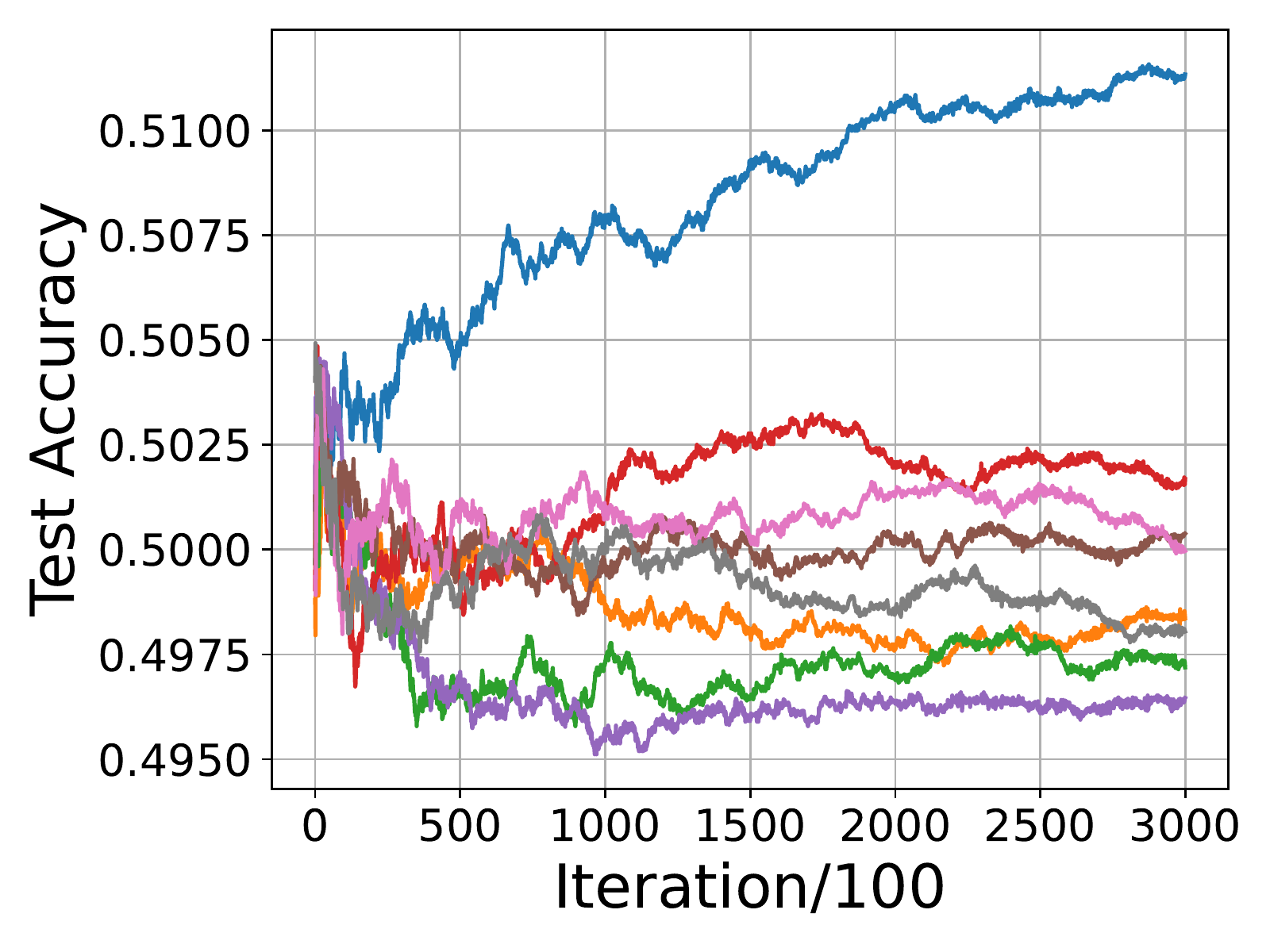}
        \caption{}
    \end{subfigure}
    \begin{subfigure}{0.3\textwidth}
        \centering
        \includegraphics[scale = 0.23]{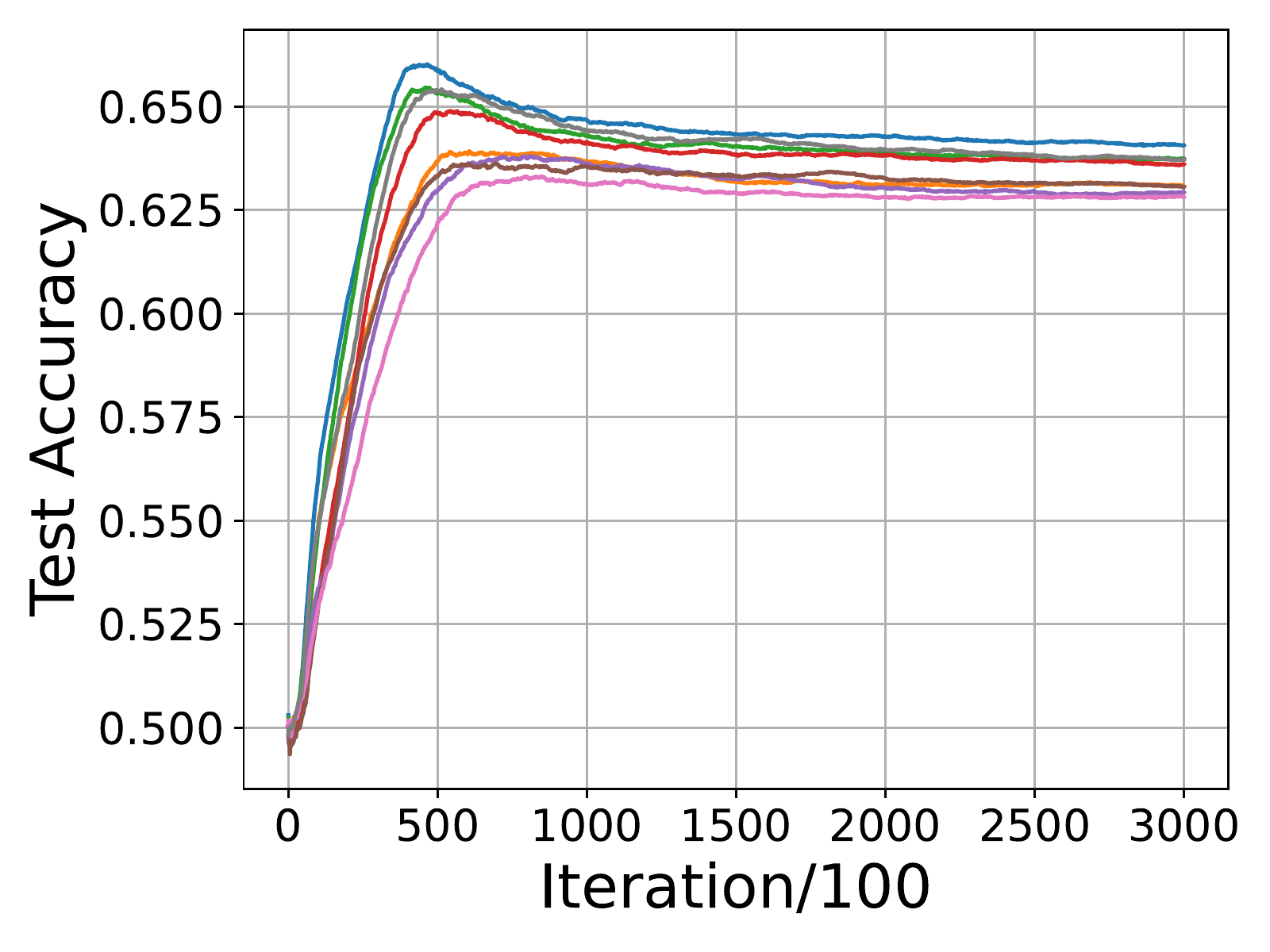}
        \caption{}
    \end{subfigure}
    \begin{subfigure}{0.3\textwidth}
        \centering
        \includegraphics[scale = 0.23]{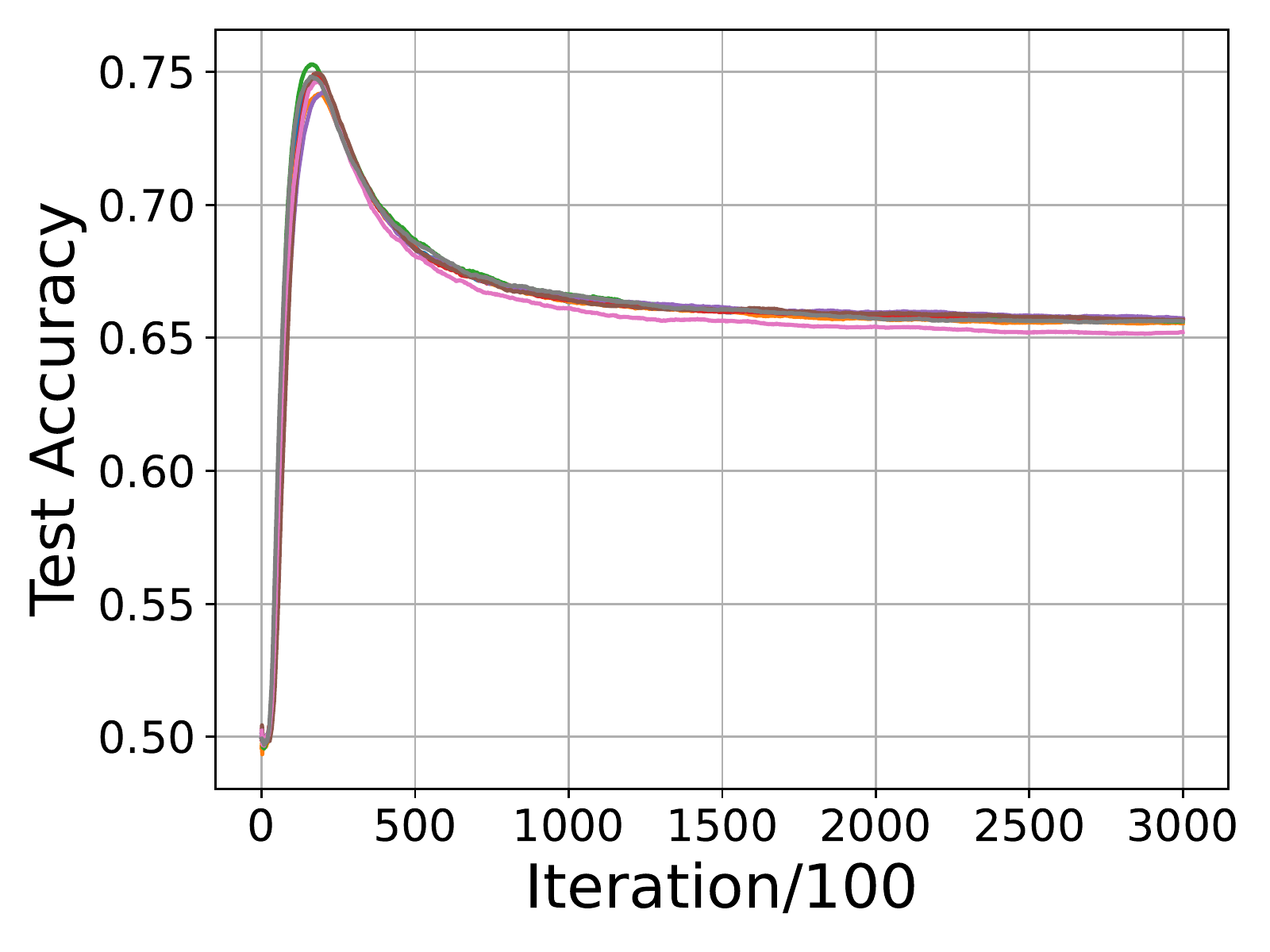}
        \caption{}
    \end{subfigure}
    \begin{subfigure}{0.3\textwidth}
        \centering
        \includegraphics[scale = 0.23]{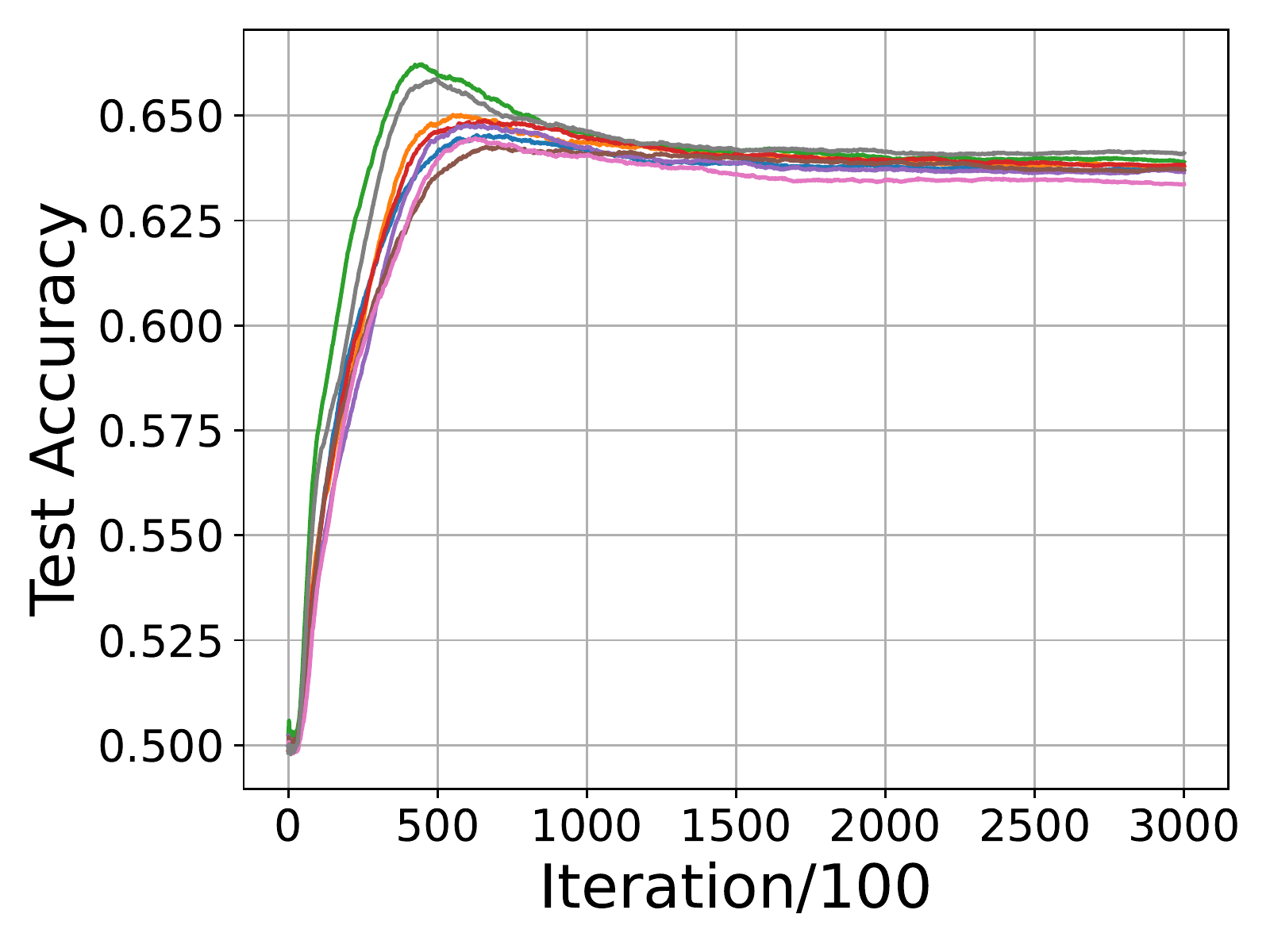}
        \caption{}
    \end{subfigure}
    \begin{subfigure}{0.3\textwidth}
        \centering
        \includegraphics[scale = 0.23]{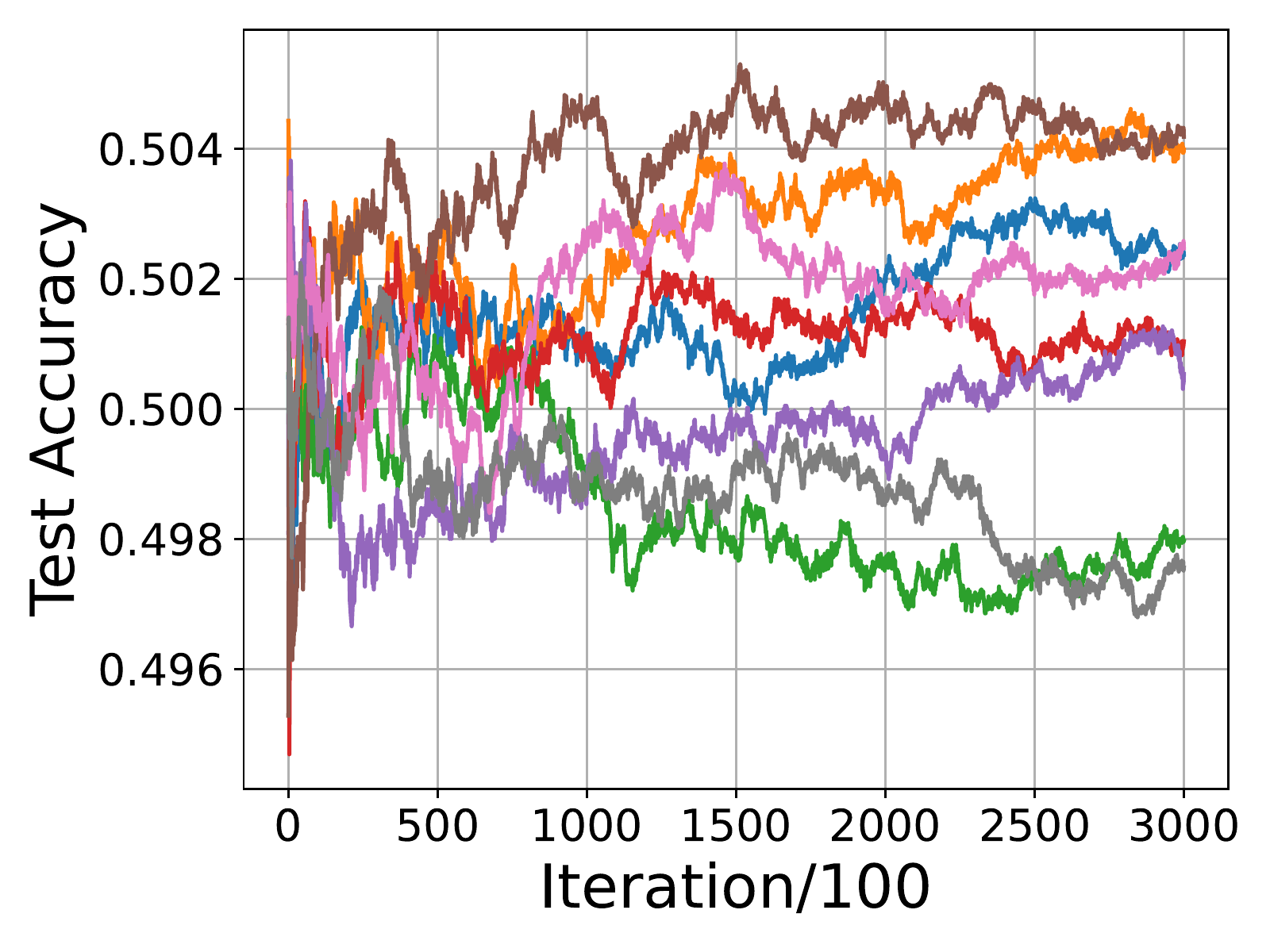}
        \caption{}
    \end{subfigure}
    \begin{subfigure}{0.3\textwidth}
        \centering
        \includegraphics[scale = 0.23]{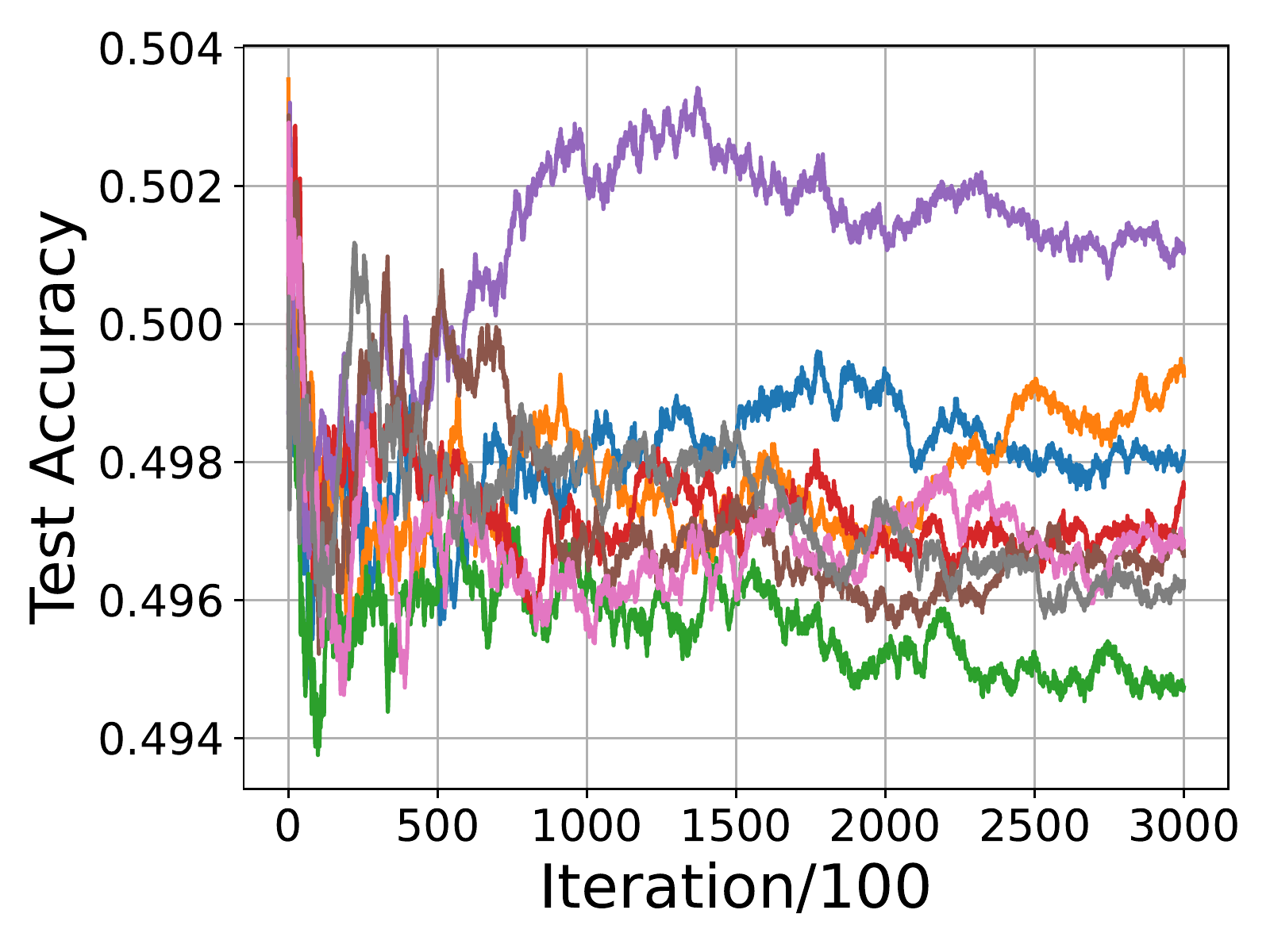}
        \caption{}
    \end{subfigure}
    \begin{subfigure}{0.3\textwidth}
        \centering
        \includegraphics[scale = 0.23]{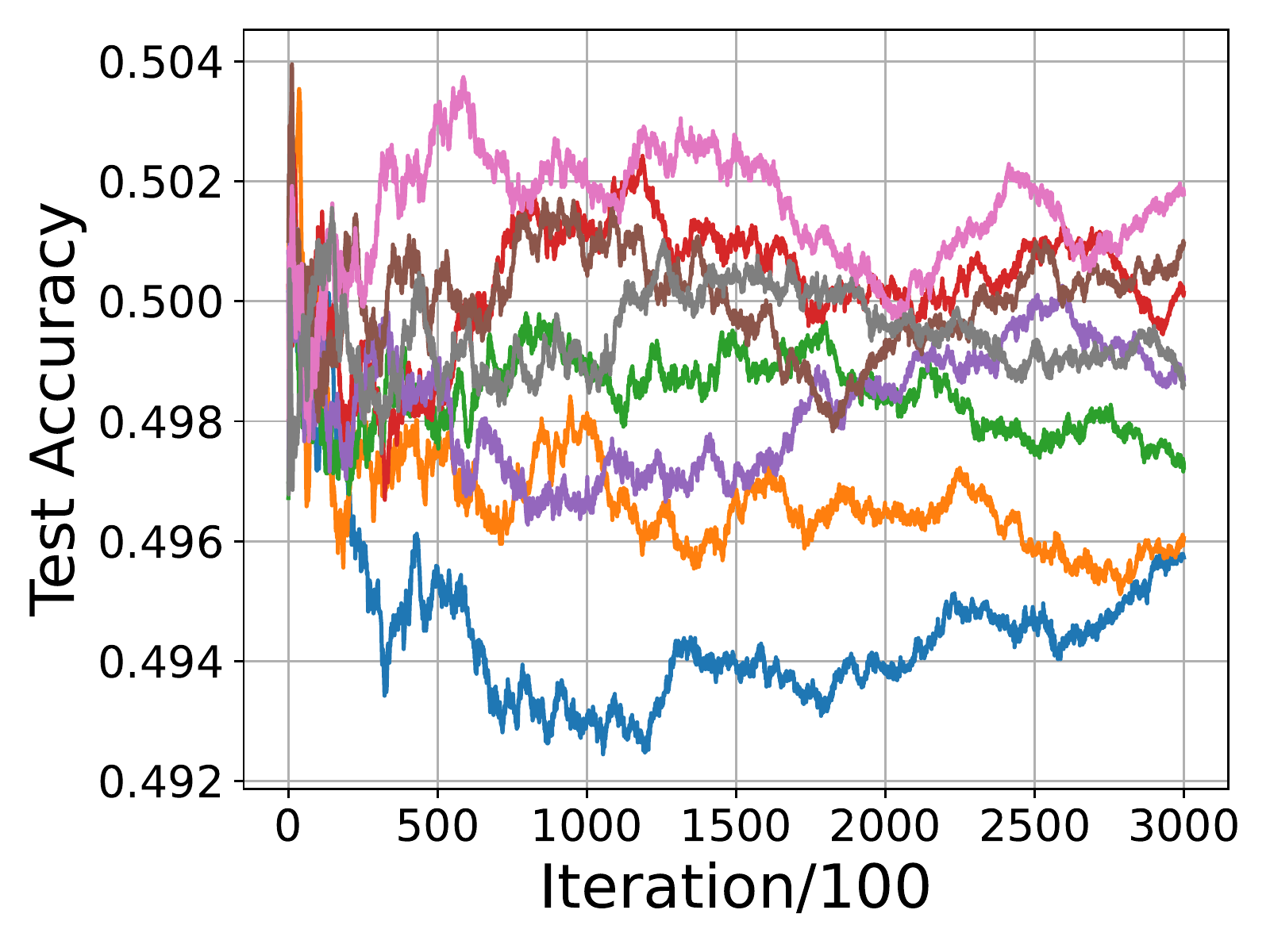}
        \caption{}
    \end{subfigure}
    \caption{\textbf{Performance of the Models} These nine graphs show the performance of models during training under three different setups specified in~\Cref{tab:restricted_setup}. The first, second, and third row corresponds to Setup 1, 2, and 3 respectively. Different pictures in the same row correspond to different datasets. Different colors in the same picture correspond to different initializations of the networks.} 
    \label{fig:naive}
\end{figure}

\begin{table}[h]
    \centering
    \begin{tabular}{|c|c|c|c|c|}
        \hline
        Setup & $n$  & $\tau$ & $m$ & Sparsity\\
        \hline
        1 & 44   & $0.2$ & $2^{17}$  & $0.2$ \\
        \hline
        2 & 29  &  $0.2$ & $2^{17}$  & $0.2$\\
        \hline
        3 & 30  & $0.2$ & $2^{17}$  & $0.5$ \\
        \hline
    \end{tabular}
    \caption{\textbf{Three Different Setups Shown in~\Cref{fig:naive}}. The sparsity column means the hamming weight of the secret over $n$.}\label{tab:restricted_setup}
\end{table}

\begin{figure}[h]
    \centering
    \begin{subfigure}{0.45\textwidth}
    \centering
        \includegraphics[scale = 0.35]{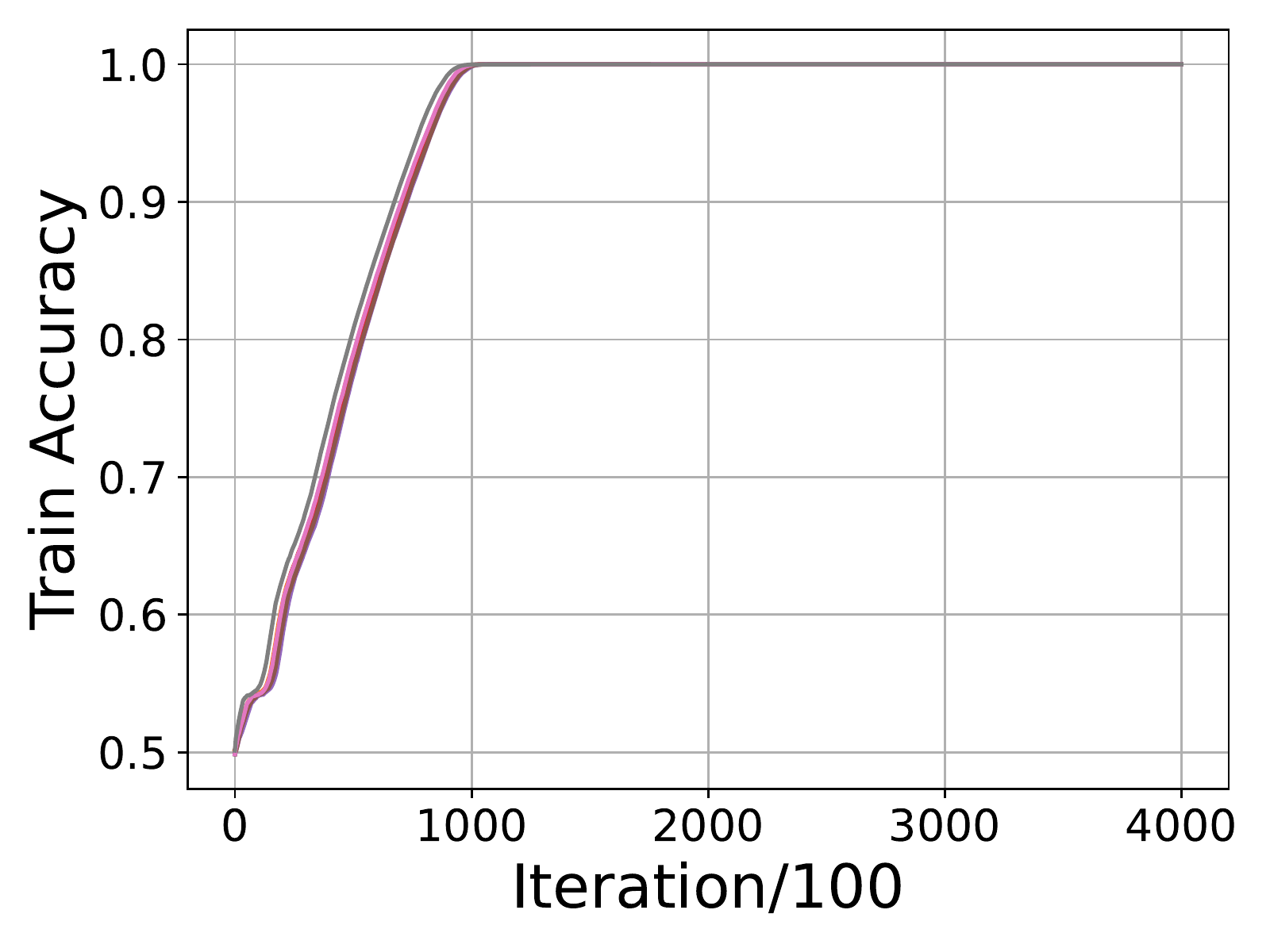}
        \caption{Training Accuracy}
    \end{subfigure}
    \begin{subfigure}{0.45\textwidth}
    \centering
        \includegraphics[scale = 0.35]{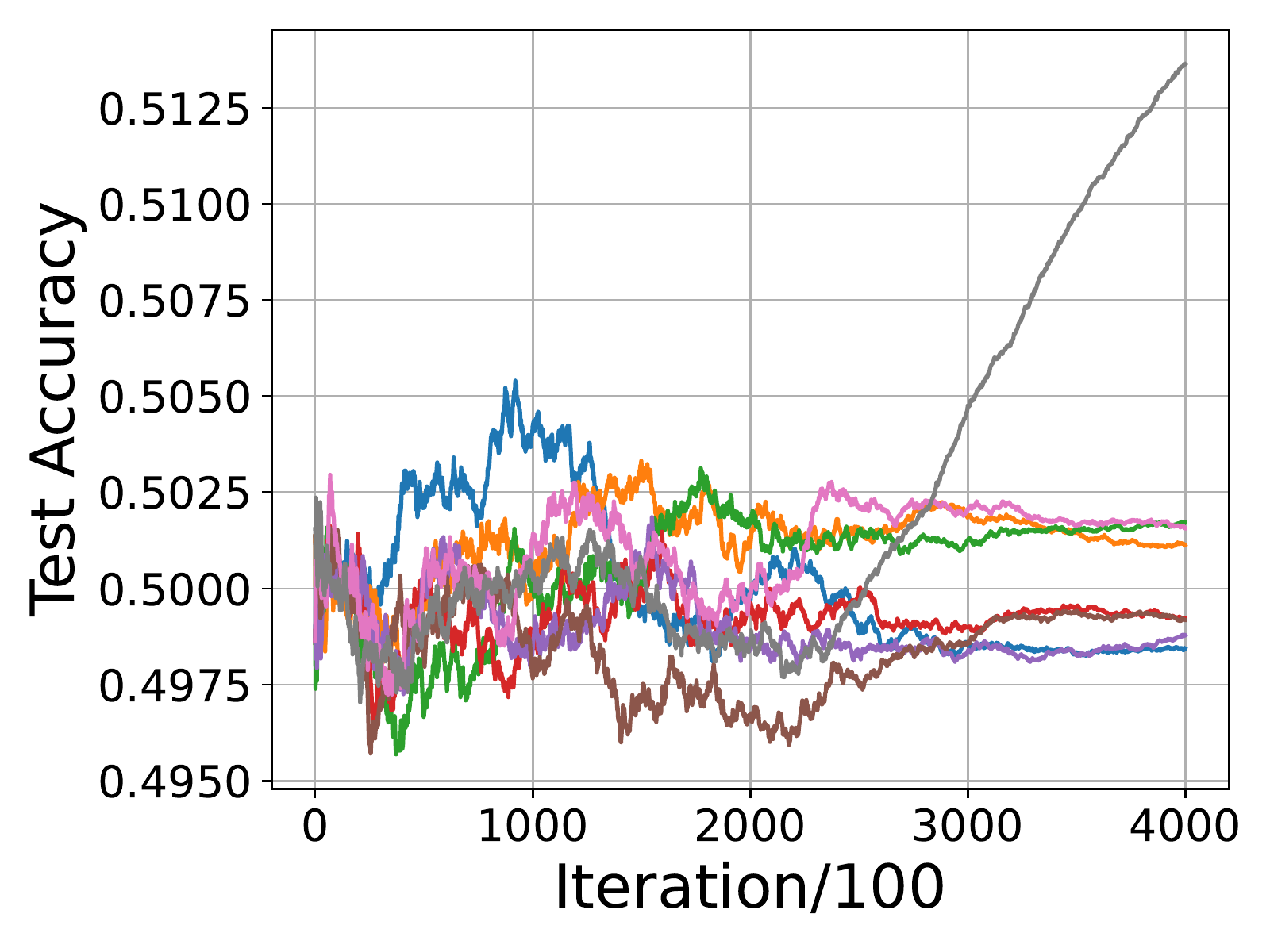}
        \caption{Test Accuracy}
    \end{subfigure}
    \caption{\textbf{Experiments on $\text{LPN}_{44,185362,0.2}$} The legend in this figure is the same as in~\Cref{fig:abundant_case}. We can observe a large deviation between the training and test accuracy curve. In fact, for the initialization that the accuracy boosts the highest, the accuracy increases after all the training data are correctly predicted, or in other words, memorized.}. 
    \label{fig:restricted_case}
\end{figure}

We first illustrate the common phenomena observed under~\Cref{setting:restricted} that guide us in designing~\Cref{alg:restricted}.

We conducted experiments with neural networks using~\Cref{alg:optim_example} under three setups specified in~\Cref{tab:restricted_setup}. The hyperparameters follows~\Cref{tab:hyperparameter,tab:shared_hyperparameter}, and the step threshold $\step$ is set to $300k$. To test the effectiveness of using sparse secrets (cf.~\Cref{lem:s_sparse}), we vary the sparsity of the secret (Hamming weight / dimension) in the third setup. We plot the test accuracy of the model with respect to training iterations in~\Cref{alg:abundant}. Notice here we test on LPN data directly hence the highest possible population accuracy is $80\%$.

\begin{Empirical}
    Under~\Cref{setting:restricted}, we find that (i) The sparse secret makes training easier. (ii) Initialization significantly affects the performance of the converged trained model. (iii) The reduction to the decision problem enables us to learn LPN instances with larger $n$. (iv) Typically there is a significant gap between training and test accuracy under this setting, which calls for methods to improve generalization.
\end{Empirical}

\paragraph{Sparsity of Secret} Comparing the first and the last row of~\Cref{fig:naive}, Algorithm \ref{alg:restricted} succeeded in solving $\LPN_{44,2^{17},0.2}$ with secrets of sparsity $0.2$ for two out of three datasets, while completely fail on solving $\LPN_{30,2^{17},0.2}$ with secrets of sparsity $0.5$. Hence, using sparse secret greatly reduce the hardness of learning. This implies we should always apply~\Cref{lem:s_sparse} to reduce the secret to a sparse secret when the noise level is low.

\paragraph{Initialization Matters} Graph (a),(b), and (c) in~\Cref{fig:naive} show that initialization can affect whether a neural network can distinguish the correct and the wrong guesses. This shows the necessity to try different initialization on the same dataset (possibly in parallel).

\paragraph{Necessity to Transfer to Decision Problem} Previously we have established that Algorithm \ref{alg:optim_example} can solve the decision version of $\LPN_{44,m,0.2}$. Graphs (d), (e), (f) show that even if we reduce $n$ to 29, direct learning cannot yield a model that perfectly simulates the LPN oracle without noise, in which case the test performance should get close to 0.8. However, all the trained models only yield an accuracy lower than $0.7$. Hence with the reduction, we can solve an LPN instance with larger $n$ with the same $m$ and $\tau$.

\paragraph{Divergence of Training and Test Accuracy} In~\Cref{fig:restricted_case}, we plot a run of~\Cref{alg:optim_example} on LPN data when the sample complexity is greatly limited. Because the size of the training set is very limited, the model easily reaches $100\%$ accuracy on the training set. This phenomenon is known as \emph{overfitting} in the machine learning community. However, one can observe that the test accuracy of one of the model in (b) boost after the model overfits the training set. Though strange at first sight, this phenomenon, known as \emph{grokking} where validation accuracy suddenly increases after the training accuracy plateau is common in deep learning experiments and repetitively appears in our experiments on LPN data. The exact implication of this phenomenon is still opaque. At the very least, we can infer from the figures that \emph{generalization}, instead of optimization is the key obstacle in~\Cref{setting:restricted}.

\subsubsection{Hyperparameter Selection}
\label{sec:restricted_hyper}

\begin{table}[t]
\centering
\begin{subtable}[t]{0.2 \textwidth}
\centering
    \begin{tabular}{ccc}
        \toprule
        1 & 2 & 3\\
        \midrule
        17 & 19 & $>$19\\
       \bottomrule
    \end{tabular}
    \caption{Depth}
    \label{tab:ablation_3_depth}  
\end{subtable}
\begin{subtable}[t]{0.4 \textwidth}
\centering
    \begin{tabular}{cccccc}
        \toprule
        500 & 800 & 1000 & 1200 & 1500 & 2000\\
        \midrule
        $>$19 & 17.5 & 17 & 17.5 & 18.5 & 17.5\\
        \bottomrule
    \end{tabular}
    \caption{Width}
    \label{tab:ablation_3_width}
\end{subtable}
\begin{subtable}[t]{0.30 \textwidth}
    \centering
    \begin{tabular}{ccc}
        \toprule
        ReLU & Sigmoid & COS\\
        \midrule
        17 & 18.5 & 18 \\
        \bottomrule
    \end{tabular}
    \caption{Activation}
    \label{tab:ablation_3_activation}
\end{subtable}

\begin{subtable}[t]{0.15 \textwidth}
    \centering
    \begin{tabular}{cc}
        \toprule
        MSE & Logistic \\
        \midrule
        18 & 17 \\
        \bottomrule
    \end{tabular}
    \caption{Loss}
    \label{tab:ablation_3_loss}
\end{subtable}
\begin{subtable}[t]{0.46 \textwidth}
    \centering
    \begin{tabular}{cccccc}
        \toprule
        0 & $5e-4$ & $1e-3$ & $2e-3$ & $3e-3$ & $6e-3$ \\
        \midrule
        17 & 17 & 16.5 & \textbf{16} & 16.5 & $>19$ \\
        \bottomrule
    \end{tabular}
    \caption{L2 regularization}
    \label{tab:ablation_3_l2}
\end{subtable}
\begin{subtable}[t]{0.25 \textwidth}
    \centering
    \begin{tabular}{ccc}
        \toprule
        Full & Full/2 & Full/4 \\
        \midrule
        17 & 17 & 17 \\
        \bottomrule
    \end{tabular}
    \caption{Batch Size}
    \label{tab:ablation_3_batch}
\end{subtable}

\begin{subtable}[t]{0.8 \textwidth}
    \centering
    \begin{tabular}{ccccccccc}
        \toprule
        $1e-5$ & $2e-5$ & $1e-4$ & $2e-4$ & $1e-3$ & $2e-3$ & $1e-2$ & $2e-2$ & $1e-1$\\
        \midrule
        18.5 & 17 & 17 & 18 & 18 & 18.5 & 18 & 18.5 & 19\\
        \bottomrule
    \end{tabular}
    \caption{Learning Rate}
    \label{tab:ablation_3_lr}
\end{subtable}
\caption{\textbf{Time Complexity w.r.t. Multiple Hyperparameters.} Each entry represents the logarithm of sample complexity with base $2$ for~\Cref{alg:restricted} with corresponding hyperparameters to solve $\text{LPN}_{44,m,0.2}$ with probability approximately $2/3$. The hyperparameter profile we considered in our experiments always shares all except the investigated hyperparameter as~\Cref{tab:shared_hyperparameter,tab:base_hyperparameter}.}
\label{tab:ablation_3}
\end{table}

In this subsection, we fix an LPN problem setup with $n = 44$ and $\tau = 0.2$ and illustrate the conclusion of our hyperparameter selection process based on~\Cref{alg:restricted_meta}. Recalling our goal is to find hyperparameters that minimize the sample complexity, we report $\log_2 \text{\# Sample}$ for our reduction algorithm to successfully return the secret with probability approximately equal to $2/3$ following the convention of~\cite{bogos2016solving}. The hyperparameter profiles we considered in our experiments are identical except for the investigated hyperparameter as~\Cref{tab:shared_hyperparameter,tab:base_hyperparameter}. The aggregated results are shown in~\Cref{tab:ablation_3}. 

\begin{Empirical}
Under~\Cref{setting:restricted}, the architecture of the model, learning rate, and L2 regularization are the most important components in determining the sample complexity. In the meantime, tuning batch size or applying other regularization methods are in general not helpful.
\end{Empirical}

\begin{table}[]
    \centering
    \begin{tabular}{|c|c|c|c|c|} 
    \hline  Sampler & Learning Rate & Weight Decay &  Batch Size &Step Threshold $\step$ \\
    \hline
      Batch & $2e-5$ & 0 & Size of Train Set & $300k$ \\
    \hline
    \end{tabular}
    \caption{Base Hyperparameters}
    \label{tab:base_hyperparameter}
\end{table}

\paragraph{Depth of Network}  In contrast to the trend of using \emph{deep} neural network, the depth of the network on LPN problem should not be large. Based on~\Cref{tab:ablation_3_depth}, we find depth $1$ neural network performs significantly better than any larger depth network.

\paragraph{Width of Network} Similar to the case in~\Cref{sec:abundant}, width of the network is critical, as shown in~\Cref{tab:ablation_3_width}. However, the blessing here is that the width is not very versatile in this setting. The width must exceed a lower bound for better performance while all the widths larger than that show similar performance. However, the optimal width is still a finite number, in this case, near $1000$.

\paragraph{Activation Function} We experiment with replacing the activation function of the first layer and find that while all the activation function can make~\Cref{alg:restricted} work given enough sample, the ReLU activation performs better than other candidates. The results are shown in~\Cref{tab:ablation_3_activation}.

\paragraph{Loss} We experiment with two types of loss, the MSE loss and the logistic loss in~\Cref{tab:ablation_3_loss}. Training using these two loss functions are usually referred to as \emph{regression} and \emph{classification} in machine learning literature. We find out that logistic loss performs slightly better than MSE loss.

\paragraph{Batch Size} One may naturally wonders why in~\Cref{setting:moderate} and~\Cref{setting:abundant}, we both use fix batch size, while in the~\Cref{setting:restricted}, we recommend using full batch training. One important issue to notice is that the batch size in the two other settings is typically larger than the total dataset itself in the current setting. In our experiments with smaller batch sizes,  the results are almost identical to those using the full training set size.

\paragraph{Learning Rate} We observe a similar phenomenon as in~\Cref{sec:abundant_hyper} for the learning rate, i.e., it should be neither too large nor too small to achieve the best performance. However, it is worth noticing that the best learning rate in this regime is typically a magnitude smaller than the best learning rate for~\Cref{alg:abundant}. We have two hypotheses for the reason behind this. First, loosely based on our theory prediction in~\Cref{thm:optimization}, a higher noise rate calls for a higher learning rate. Second, the dependency between the model weight and the sample we used for each step calls for a smaller learning rate to reduce the caveat of overfitting.

\paragraph{Weight Decay} As mentioned in~\Cref{sec:restricted_case}, generalization is the key obstacle for neural networks to learn to distinguish LPN data from random data in~\Cref{setting:restricted}. We experiment with several different regularization functions and find out that weight decay, or L2 regularization, is of significance in reducing the sample size. This coincides with previous theoretical prediction~\Cref{thm:wd} in a simpler setting. The best weight decay parameter we find for our settings is $2e-3$. We observe empirically that when increasing this parameter to $6e-3$, the weight of the model will quickly collapse to all zeroes. In this sense, the weight decay parameter needs to be large, but not to the extent that interferes with optimization.

\paragraph{Sparsity Regularization} We also try regularization methods like L1 regularization and architectures with hard-coded sparsity (by fixing some of the parameters to $0$ or using a convolutional neural network), the results are not shown in~\Cref{tab:ablation_3} because all these methods increase the required sample complexity by a large amount. Although in our~\Cref{thm:rep}, the ground truth model we constructed is sparse when the secret is sparse, the ways we utilize sparsity hurt the optimization performance of~\Cref{alg:optim_example}. It remains open whether one can design better architectures or find other regularization functions that better suit the LPN problems.

\subsection{Moderate Sample Setting}\label{sec:moderate}

In~\Cref{setting:moderate}, we aim at finding ways to use neural networks as a part of the classical \emph{reduction-decoding} scheme of LPN solving. We have observed in~\Cref{sec:abundant} that neural networks show resilience to noise when the dimension is small, which is further validated by~\Cref{thm:optimization}. It would then be tempting to use neural networks in the decoding phase where typically the dimension is low and the noise rate is high. However, to make neural networks feasible for solving the problem after reduction, we can no longer assume the number of training samples is infinite. Instead, we should utilize all the training samples generated after the reduction phases and try reducing the time complexity based on the samples we are given.

In this section, we use~\Cref{alg:moderate} to solve for the last $26$ bits of the secret for LPN instance  $\text{LPN}_{125,1.2e10,0.25}$ in \textbf{106 minutes}. As a comparison, in~\cite{DBLP:conf/crypto/EsserKM17}, more than 3 days are taken to recover the $26$ bits of the secret. We believe this result shed light on the potential of utilizing neural networks and GPUs as components of faster solving algorithms of LPN in practice.

In~\Cref{sec:moderate_case}, we show the exact time component for solving the last $26$ bits as well as important phenomenons we observe under~\Cref{setting:moderate}. In~\Cref{sec:moderate_hyper}, we mainly discuss the hyperparameter tuning required for the post-processing part of~\Cref{alg:moderate}. We will show that the hypothesis testing threshold $\tau'$ is not very versatile for the algorithm performance and hence is easy to tune.

\subsubsection{Case Study}
\label{sec:moderate_case}

We will first introduce how we solve for the last $26$ bits of the LPN instance $\text{LPN}_{125,1.2e10,0.25}$. On our server with 128 cores, we can reduce this problem to $\text{LPN}_{26, 1.1e8, 0.498}$ in 40 minutes. We will then enumerate the last $6$ bits of the secret to further reduce the problem to $\text{LPN}_{20, 1.1e8, 0.498}$ and apply our~\Cref{alg:moderate} on it. We specified the time constraint $\rtime$ in~\Cref{tab:hyperparameter} for~\Cref{alg:moderate} as $20$ minutes and hypothesis testing threshold $\tau'$ as $0.483$ under this setting. We further restricted the running time of the rebalance and post-processing step by $2$ minutes. By using $8$ 3090 GPUs to enumerate secret in parallel, we try the correct postfix in the third round of enumeration and the post-processing step return the right secret in the pool of secrets. The applications of enumeration and~\Cref{alg:moderate} take less than $66$ minutes. In total, we solve for the last $26$ bits in less than $106$ minutes. 

In designing~\Cref{alg:moderate}, two primary factors affect the performance. The first factor is the probability of~\Cref{alg:optim_example} procedure returns model with high accuracy on test data, which we dub as \emph{success rate}. The second factor is the accuracy the model can reach given the time limit under such a scenario, which we dub as \emph{mean accuracy}.  The repeat number for starting with different initializations $repeat$ mainly depends on the success rate and the mean accuracy affects the time complexity of our post-processing step. We will now introduce some important phenomenons we observe in applying~\Cref{alg:optim_example} on LPN instance with dimension $20$ and error rate $0.498$ that guides us in selecting the hyperparameters of~\Cref{alg:moderate}.

\begin{Empirical}
    Under~\Cref{setting:moderate}, we find that (i) Both the success rate and the mean accuracy increases with the size of the training set. (ii) An imbalance in the output distribution exists in all models we trained in this setting (predicting one outcome with a significantly higher probability than the other over random inputs) (iii) Different than the case in~\Cref{setting:restricted}, the training accuracy typically plateau at a low level.
\end{Empirical}

\paragraph{The Size of Training Set}
We first show sample complexity affects these two factors mutually.
As most of the samples consumed are used to perform~\Cref{alg:optim_example}, we experiment with varying the sample in~\Cref{tab:ablation_3_m}. It is observed that both the success rate and the mean accuracy increase with the sample complexity. Under the scenario where $1e8$ samples are provided, the success rate reaches $87.5\%$, allowing us to set $repeat$ as 1 to get a final algorithm that succeeds with high probability. Given the randomness here is over the initialization of the model instead of the dataset itself, our results also show that we can compensate for the lack of samples by trying more initializations. 

\begin{table}[h]
    \centering
    \begin{tabular}{ccc}
        \toprule
        $m$ & \quad Success Rate & \quad Mean Accuracy \\
        \midrule 
        $4e7$ & $25.0\%$ & $51.1\%$\\
        $6e7$ & $62.5\%$& $51.7\%$ \\
        $8e7$ & $62.5\%$& $51.8\%$\\
        $1e8$& $87.5\%$ & $52.2\%$ \\
        \bottomrule
    \end{tabular}
    \caption{\textbf{Success Rate and Mean Accuracy w.r.t. Training Set Size $m$ on $\text{LPN}_{20,m,0.498}$}.
    We define a successful run as a model with random initialization that converges to a model with an accuracy greater than $51.5\%$ on clean data in the time limit. The success rate means the probability of a run being successful over the randomness of initialization. The mean accuracy is only calculated on successful runs.
    We clearly observe that with an increasing number of samples, both the success rate and the mean accuracy increase significantly.}
    \label{tab:ablation_3_m}
    \vspace{-0.3in}
\end{table}

\paragraph{Discussion on Output Distribution} 
Reader may find the rebalance step in~\Cref{alg:moderate} seemingly unnecessary. However, we observe in experiments that the trained neural network typically has a bias towards $1$ or $0$. The proportion of two different results given random inputs can be as large as $1.5$. This fact makes the testing of Gaussian Elimination steps hard and leans toward providing a wrong secret. We mitigate this effect by performing the rebalance step and believe that if this phenomenon can be addressed through better architecture designs or training methods, the performance of neural network based decoding method can be further boosted.

\paragraph{Optimization Obstacle} We plot the training and test accuracy curve for applying~\Cref{alg:optim_example} over $\text{LPN}_{20,1e8,0.498}$ in~\Cref{fig:moderate_case}. Here the test is performed on clean data to provide better visualization. Different from~\Cref{setting:restricted}, where the training set is greatly limited and the network can overfit the training set, here we observe that the network can hardly fit the training data.
This fact has two sides. Firstly, it shows that overfitting ceases given enough samples and generalization won't be a severe issue. However, as we observe in~\Cref{fig:restricted_case}, the generalization performance may boost when overfitting happens, hence the low training accuracy may also be a reason behind the low final converged accuracy.

\begin{figure}[h]
    \centering
    \begin{subfigure}{0.49\textwidth}
    \centering
        \includegraphics[scale = 0.35]{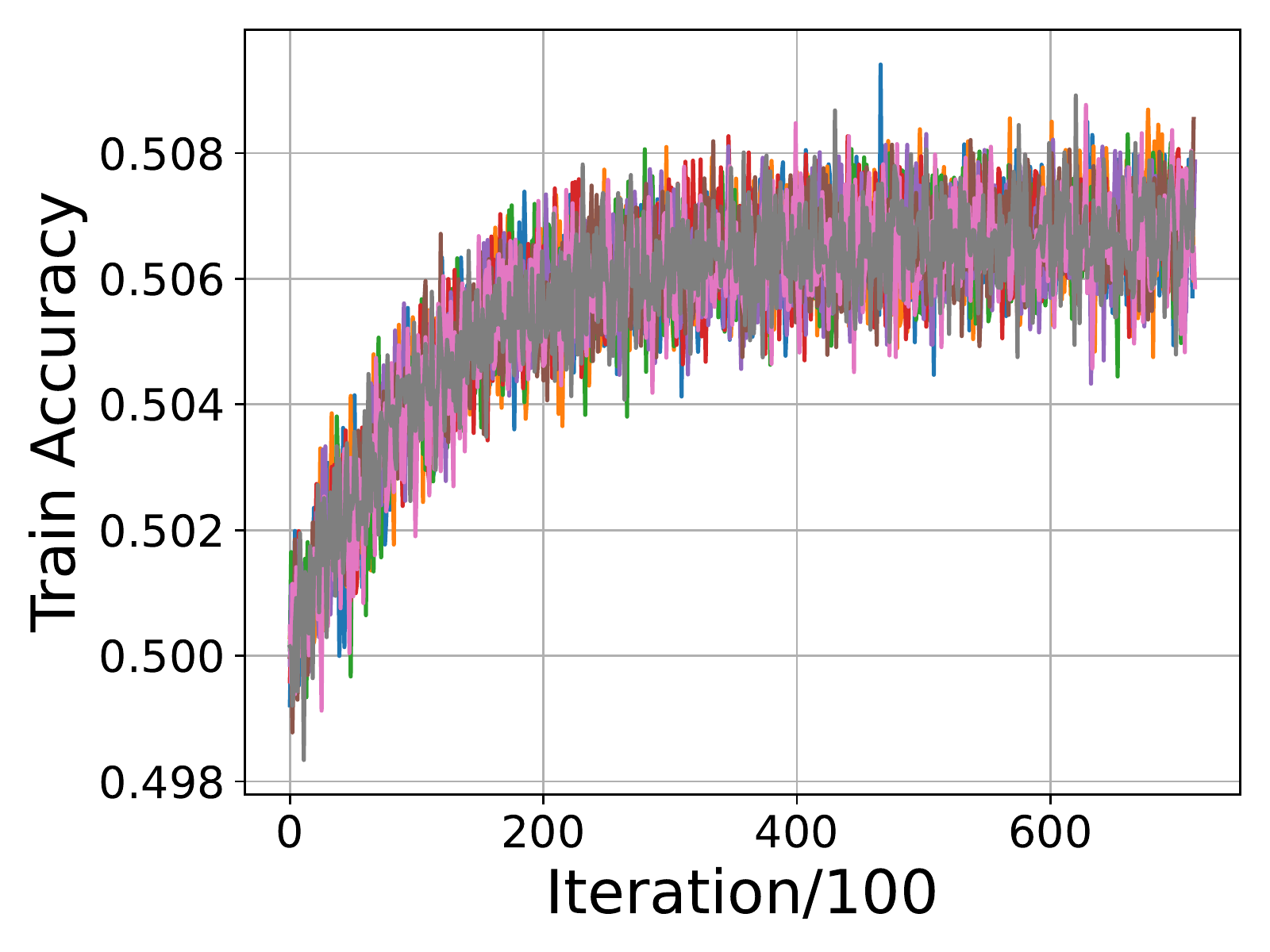}
        \caption{Training Accuracy}
    \end{subfigure}
    \begin{subfigure}{0.49\textwidth}
    \centering
        \includegraphics[scale = 0.35]{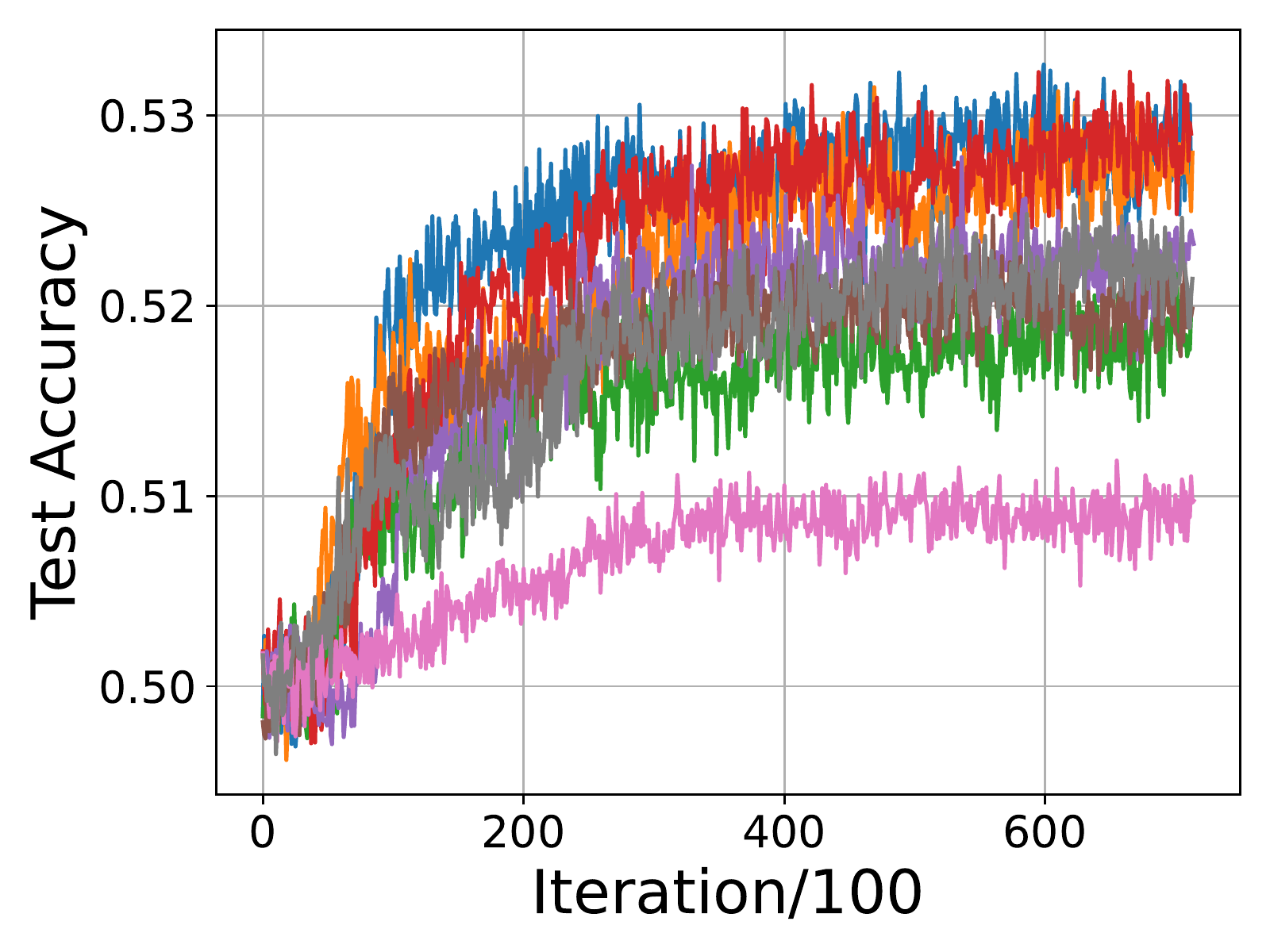}
        \caption{Test Accuracy}
    \end{subfigure}    \caption{\textbf{Experiments on $\text{LPN}_{20,1e8,0.498}$} We observe that the training accuracy is low, showing a vast contract with~\Cref{setting:restricted}. One should notice here the test accuracy is over clean data while the training accuracy is over the noisy training set.}. 
    \label{fig:moderate_case}
\end{figure}

\subsubsection{Hyperparameter Selection}
\label{sec:moderate_hyper}

For the ease of hyperparameter tuning, we use hyperparameters for~\Cref{alg:optim_example} we found in~\Cref{sec:abundant_hyper} and our experiments show that this choice can already let us get models with enough accuracy to perform boosting. Hence here we will focus on discussing the hyperparameters for the post processing the Gaussian Elimination step.

\begin{Empirical}
Under~\Cref{setting:moderate}, we find that (i) Hyperparameters for the network training can follow the hyperparameters under~\Cref{setting:abundant}. 
(ii) The hypothesis testing threshold $\tau$ in~\Cref{alg:moderate} has large tolerance and has little impact on time complexity even when it's $1\%$ larger than the ground-truth error rate in the test set for the post processing step.
\end{Empirical}

\paragraph{Hypothesis Testing Threshold} In the post processing step, the samples are labeled by the neural networks and are in this sense \emph{free}. Under our experiments, we choose the size of boosting set $m' = 231072$, with $131072$ samples as the sample pool where Gaussian elimination input are sampled and $100000$ samples for hypothesis testing, which are more than sufficient for the estimated error rate $\approx 0.48$.

The remaining question for hyperparameter tuning may be that how should we set the hypothesis testing threshold given that it may be hard to estimate on the fly. We propose to use a meta run to estimate the converged accuracy. As one can observe that from~\Cref{fig:moderate_case}, the final converged accuracy of successful run (with final accuracy greater than $51.5\%$) center closely. 

We further show through experiments that the hypothesis testing threshold $\tau'$ in the Gaussian step are not critical for the performance of~\Cref{alg:moderate} in~\Cref{tab:ablation_3_tau}. The estimation of hypothesis threshold can be off by almost $1\%$ while still returning the correct secret in $20$ repetitive pooled Gaussian run. However, the results in this table also shows that the final testing on noisy real data is necessary as all the pools of secrets contain wrong secrets. The reader should notice that the final testing step takes almost negligible time, as for each model, at most $20$ secrets need to be tested.

\begin{table}[h]
    \centering
    \begin{tabular}{cccc}
        \toprule
         0.480 & 0.483  & 0.486 & 0.489 \\
        \midrule 
        6 & 6 & 1 & 2  \\
        \bottomrule
    \end{tabular}
    \caption{\textbf{Number of Occurrences of the Correct Secret in 20 runs w.r.t. Estimate Error Rate $\tau'$} A subtlety in hyperparameter tuning for~\Cref{alg:moderate} is that in general $\tau'$ is hard to know precisely. We propose to use a meta run to first estimate $\tau'$ of the converged model. As one can infer from~\Cref{fig:moderate_case}, the test accuracy of the converged model in successful runs tends to be stable. With the ground truth error rate in the test set being about $0.479$, our experiments shows that in fact the post-processing step has high tolerance for the hypothesis testing threshold, ranging from $0.480$ to $0.489$.} 
    \label{tab:ablation_3_tau}
\end{table}

\section{Theoretical Understanding}
\label{sec:theory}

In this section, we will show some primary efforts on understanding the effect of neural networks on LPN problems.
Because the general understanding of the optimization power and generalization effect of the neural network is very limited, the results in this 
section can't fully explain all the features of our algorithms and our hyperparameter choices. 
However, this theoretical analysis improves our understanding of our empirical findings and hence is recorded here.

This section is organized according to the decomposition in~\Cref{eq:decompose}. 
\Cref{sec:theory_representation} shows that the representation gap of the model architecture MLP over all loss is optimal over the whole function class $f:\R^d \to [0,1]$.
\Cref{sec:theory_optimization} shows that despite the inevitable exponential dependency on the dimension, the time complexity of~\Cref{alg:abundant_theory} scaled optimally with respect to noise.
\Cref{sec:theory_generalization} introduces some prior results that partially explain why weight decay is a powerful regularization as discovered by our ablation study in~\Cref{sec:restricted}.  \Cref{sec:theory_hardness} provides a detailed discussion on the hardness of LPN problem using gradient-based methods.

\subsection{Representation Power}
\label{sec:theory_representation}

The main result of this section is the following theorem, which shows that an MLP with a width equivalent to the input dimension is sufficient to represent the best prediction of the LPN inputs.

\begin{theorem}\label{thm:rep}
    For any continuous loss function $l: [0,1] \times \{0,1\} \to \R^+$, dimension $n$ and error rate $\tau$, there exists a weight $\weight$ for a depth 1 MLP $\model$ with width $n$, $\sigma_1 = ReLU$, $\sigma_2 = \sigmoid$, such that the representation gap $\populoss(f)$ 
    of the specified function $f = \model[\weight]$ is approximately minimized over the function class $\functionclass = \{ f':\R^d \to [0,1] \}$. Quantitatively, for any $\epsilon > 0$, there exists $\weight_{\epsilon}$, such that 
    \begin{align*}
        \populoss(\model(\weight_\epsilon)) - \min_{f \in \functionclass} \populoss(f) \le \epsilon.
    \end{align*}
\end{theorem}

\begin{proof}\label{proof:rep}
    For LPN problems with dimension $n$, secret $s$ and error rate $\tau$, it holds that 
    \begin{align*}
        \populoss(f) &= \E_{x \sim U(\Z_2^n)}[\tau \ell(f(x), s^t x \mod 2) + (1 - \tau) \ell(f(x), (s^t x + 1) \mod 2)] \\
        &\ge \E_{x \sim U(\Z_2^n)}[\min_{p \in [0,1]} \tau \ell(p, s^t x \mod 2) + (1 - \tau) \ell(p, (s^t x + 1) \mod 2) ]. \\
    \end{align*}

    We will now show that in fact this lower bound can be reached approximately when $f$ is specified by a weight $\weight$ for a depth 1 MLP $\model$ with width $n$, $\sigma_1 = \relu$, $\sigma_2 = \sigmoid$.

    By~\Cref{lem:rep}, there exists a weight $\weight'$  for a depth 1 MLP $\model'$ with width $n$, $\sigma_1 = \relu$, $\sigma_2 = \identity$ such that
    \begin{align*}
        \forall x \in \{0,1\}^n, \model'[\weight'](x) = s^tx \mod 2.
    \end{align*}

    Given $l$ is continuous, for $b \in \{0,1\}$, we can find $a_{b,\epsilon} \in (0,1)$, such that
    \begin{align*}
        \tau \ell(a_{b,\epsilon}, b) + (1 - \tau) \ell(a_{b,\epsilon}, 1 - b) \le \min_{p \in [0,1]} \ell(p, b) + (1 - \tau) \ell(p, 1 - b)  + \epsilon.
    \end{align*}

    Further define $\gamma_{b,\epsilon}$ satisfies $\sigmoid(\gamma_{b,\epsilon}) = a_{b, \epsilon}$.
    For $\weight$ satisfying
    \begin{align*}
        \weight_1 &= \weight'_1, \\
        b_1 &= b'_1, \\
        \weight_2 &= (\gamma_{1,\epsilon} - \gamma_{0,\epsilon})\weight'_2, \\
        b_2 &= \gamma_{0,\epsilon},
    \end{align*}
    it holds that $\model[\weight](x) = a_{s^tx \mod 2, \epsilon}$. This implies $\populoss(\model[\weight]) \le \epsilon + \min_{f \in \functionclass} \populoss(f)$. The proof is complete.
\end{proof}

The above proof relies on~\Cref{lem:rep}, whose proof is as follows.

\begin{lemma}\label{lem:rep}
    For any secret $s \in \{0,1\}^n$, there exists a weight $\weight$ for a depth 1 MLP $\model$ with width $n$, $\sigma_1 = \relu$, $\sigma_2 = \identity$, such that
    \begin{align*}
        \forall x \in \{0,1\}^n, \model[\weight](x) = s^tx \mod 2  
    \end{align*}
\end{lemma}

\begin{proof}
    Recall in~\Cref{def:mlp}, the weight $\weight$ is a list of two matrices tuples specifying the affine transformation on each layer. 
    
    We will choose $\weight_1 = es^t$ with $e$ as the all $1$ $n-$dimensional vector, $b_1 = [0, -1 ,...,-n + 1]^t$. Then we would have 
    \begin{align*}
        \sigma_1(T[\weight_1, b_1](x)) = \sigma_1((s^tx) e - b_1) = [\max\{(s^tx) - i + 1, 0\}]^t_{i \in [1:n]}. 
    \end{align*}

    We will further define $b_2 = 0$ and recursively $\weight_2 \in \R^n$ as 
    \begin{align*}
        &\weight_{2,1} = 1. \\
        &\weight_{2,i} = i \mod 2 - \sum_{j = 1}^{i - 1} \weight_{2,j} (i - j + 1). \quad d -1 \ge i \ge 2
    \end{align*}

    It is then easy to check
    \begin{align*}
        T[\weight_2, b_2]\left( \sigma_1(T[\weight_1, b_1](x)) \right) = \sum_{i = 1}^n \weight_{2,i} \max\{(s^tx) - i + 1, 0\} = s^tx \mod 2.
    \end{align*}
    The proof is then complete.
\end{proof}

\subsection{Optimization Power}
\label{sec:theory_optimization}
While~\Cref{sec:theory_representation} shows that an MLP with width $n$ is sufficient to minimize the representation gap, it does not 
show how may we find such representation. As mentioned in~\Cref{sec:prelim}, it is common practice to use the gradient method to optimize 
neural networks. Pitifully,~\cite{abbe2020poly} has shown that it requires at least $n^{\Omega(\text{Hamming weight of secret})}$ time complexity to 
use any gradient method to solve LPN instance.~\cite{barak2022hidden}  shows through experimental and theoretical analysis 
that the time complexity to solve LPN instance with noise rate $0$ (known as the \emph{parity problem}) with neural networks seems to match this lower 
bound closely. In this section, we show that despite this exponential reliance on the problem dimension, the training of neural networks is robust to 
the noise rate $\tau$ in~\Cref{setting:abundant} by the following theorem.

\begin{theorem}
\label{thm:optimization}
    If there exists weight initialization $\weight_0$,  learning rate $\eta$, weight decay parameter $\lambda < 1/\eta$, and step threshold $T$ such that~\Cref{alg:abundant_theory} can return a model with 
    accuracy at least $\gamma > \frac{1}{2}$ on LPN problem with dimension $n$, secret $s$ and noise rate $0$, for any batch size $B$ greater than threshold $B_{th}$ with constant probability $p > 0$, then for any $\gamma' < \gamma, p' < p$,
    there exists another threshold $B_{th, \tau} = \max \left(O(\frac{1}{(1 - 2\tau)^2}) , B_{th} \right)$, such that~\Cref{alg:abundant_theory} can return a model with 
    accuracy at least $\gamma' \tau + (1 - \gamma')(1 - \tau)$ on LPN problem with dimension $n$, secret $s$ and noise rate $\tau$ when batch size $B \ge B_{th, \tau}$ and learning rate $\eta' = \frac{\eta}{1 - 2\tau}$ with probability $p'$ with all other hyperparameters fixed. 
\end{theorem}

As the time complexity of~\Cref{alg:abundant_theory} is $O(BT)$,~\Cref{thm:optimization} shows that with neural network can solve the LPN problem with noise rate $\tau$ in time complexity $O(\frac{1}{(1 - 2\tau)^2})$ under~\Cref{setting:abundant}, 
given the underlying parity problem can be solved by the same neural network with constant probability. This rate coincides with the sample complexity of hypothesis testing on whether a boolean vector is $s$ and is in this sense \emph{optimal}.

One would naturally ask whether corresponding results exist under the two other settings. We observe empirically that as apposed to the high converging accuracy in~\Cref{setting:abundant}, under~\Cref{setting:moderate}, the final converging model would have low train accuracy (though greater than $50\%$) despite the training time.
In~\Cref{setting:abundant}, however, the training accuracy can tend to $1$ in a short period of time while the testing accuracy increases slowly, in many cases after the convergence of training accuracy. 
These interesting phenomenons show that the optimization dynamics may be highly different under the three settings and it remains an open question to fully characterize 
how different sample complexity shapes the optimization landscape of neural networks in the LPN problem.

\begin{proof}[Proof of~\Cref{thm:optimization}]

    We will denote the weight sequence generating by applying~\Cref{alg:abundant_theory} with batch size $B$ and learning rate $\eta$ on LPN instance with dimension $n$ and noise rate $\tau$ as $W_t$.
    Assuming the corresponding batch is $\dataset_t = \{(x_{t,i}, y_{t,i})\}$ for $t \le T$. 
    We will further define $\dataset_{t,\tau} = \{(x_{t,i}, y_{t,i} + f_{t,i} \mod 2)\}$ where $f_{t,i}$ are independent boolean variables which equals to $1$ with probability $\tau$.
    Finally, we will define the coupled weight sequence $\weight_{t, \tau}$ as  weight sequence generating by applying~\Cref{alg:abundant_theory} with learning rate $\eta/(1 - 2\tau)$ and batches $\dataset_{t,\tau}$.

    Consider the following sequence of weights,
    \begin{align*}
        \tw_0 &= W_0 \\
        \tw_t &= \tw_{t-1} - \eta\lambda \tw_{t-1} - \eta \nabla_{\weight} \E_{ {x \sim U(\Z_2^n)}} l(\model[\tw_{t-1}](x), s^t x \mod 2 ), 1 \le t \le T, t\in \Z.
    \end{align*}

    By standard approximation theorem, when the batch size $B$ tends to infinity, we would have $W_t \to \tw_t$ in probability. This implies $\tw_t$ corresponds to a function with accuracy at least $\gamma$ on LPN data without noise.
    We can choose $r$ small enough such that for any weight $\weight$ in $r$ neighbors of $\tw_T$ all correspond to a function with accuracy at least $\gamma'$ on LPN data without noise.
    Suppose now when $B \ge B_r > B_{th}$, we would have with probability $1 - \frac{p - p'}{2}$, $\|\weight_{t} - \tw_t \| \le \frac{r}{2}$.
    Here $\tilde W_l$ and $\tilde b_l$ are the corresponding weight and bias of $\tilde W$.
    We would now choose a compact convex set $\mathcal{C}$ large enough such that it contains $\{\tw_t\}_{t \in [0:T]}$ and their $r-$neighbor. By our assumption, we would have for any fixed $x \in \{0,1\}^n, y \in \{0,1\}^n$, and $w \in \{W_1, W_2, b_1, b_2\}$,
    $\frac{\partial l(M[\weight](x), y)}{\partial w}$ is first-order differentiable functions, hence we can assume that there exists constant $C_1$ and $C_2$, such that,
    \begin{align*}
        &C_1 \ge \max_{\weight = ((W_1,b_1),(W_2,b_2)) \in \mathcal{C}} \max_{w \in \{W_1, W_2, b_1, b_2\}} \max_{x \in \{0,1\}^n, y \in \{0,1\}} \left\| \frac{\partial \model[\weight](x_i)}{\partial w} \right \|_{\infty} \\
        &C_2 \|W_a - W_b \|_2 \ge \max_{x \in \{0,1\}^n, y \in \{0,1\}} \left\| \frac{\partial l(\model[\weight](x),y)}{\partial \weight} \mid_{\weight = \weight_a} - \frac{\partial l(\model[\weight](x),y)}{\partial \weight} \mid_{\weight = \weight_b} \right\|
    \end{align*}

    We will first fix $\epsilon$ be a small constant such that $e^{C_2T} \epsilon \le \frac{r}{2}$ and then choose $B_{th,\tau} = \max\{\frac{2 \ln T + \ln nd - \ln(p - p')}{(1 - 2\tau)^{2}} \frac{32ndC_1^2}{\epsilon^2}, B_{r} \}$.
    When $B \ge B_{th,\tau}$, we will inductively prove that for step $t \le T$, event $E_t: \|\weight_{t,\tau} - \weight_{t}\|_2 \le e^{\eta C_2 t} t \epsilon $ happens with probability $1 - \frac{(p - p')t}{2T} - \frac{(p - p')}{2}$.
    
    We would first suppose $\|W_t - \tw_t\| \le \frac{r}{2}$ by the definition of $B_{r}$.
    Suppose $E_t$ happens, as $e^{\eta C_2 t} \epsilon \le \frac{r}{2}$, we would have the $\weight_{t,\tau} \in C$.
    By~\Cref{lem:estimate_gradient},  it holds that with probability $1 - (p - p')/2T$, for any parameter $w$ in $\weight_{t,\tau}$, it holds
    that
    \begin{align*}
        \left \| \frac{1}{1 - 2 \tau}\frac{\partial \frac{1}{B} \sum_{i} l\left(\model[\weight](x_{t,i}), f_{t,i} + y_{t,i} \mod 2 \right)}{ \partial w} - \frac{\partial \frac{1}{B} \sum_{i} l\left(\model[\weight](x_{t,i}), y_{t,i} \right)}{ \partial w} \right \|_2 \le \frac{\epsilon}{4}.
    \end{align*}
    Considering the update rule of the SGD optimizer,
    \begin{align*}
        \left \| \weight_{t+1, \tau} - \weight_{t,\tau} - \eta \lambda \weight_{t,\tau} - \eta \frac{\partial \frac{1}{B} \sum_{i}l\left(\model[\weight](x_{t,i}), y_{t,i} \right)}{ \partial \weight} \mid_{\weight = \weight_{t,\tau}} \right\|_2 \le \epsilon.
    \end{align*}

    This then implies
    \begin{align*}
        \| \weight_{t + 1, \tau} - \weight_{t + 1} \|_2 &\le 
        \epsilon + \| \weight_{t, \tau} - \weight_t \|_2 + \eta C_2 \| W_{t,\tau} - W_t\|_2 \\
        &\le \epsilon + \exp(\eta C_2) \exp(\eta C_2 t) t \epsilon \\
        &\le e^{\eta C_2 (t + 1)} (t + 1) \epsilon.
    \end{align*}
    The induction is then complete.

    Considering the induction conclusion when $t = T$ and combining the definition of $r$, the proof is complete.
\end{proof}

\begin{lemma}
\label{lem:estimate_gradient}
    With loss function $l$ being the MAE loss and $\model$ being a one-layer MLP with weight $d$ and smooth activation function, suppose $(x_i,y_i)_{i \in [1:B]}$ are i.i.d sample from LPN problem with dimension $n$ and noise rate $0$. Suppose further $f_i$ are i.i.d
    random variables following distribution $p(f_1 = 0) = 1 - p(f_1 = 1) = \tau$, for any weight $\weight = ((W_1,b_1),(W_2,b_2))$ for $\model$, it holds that for $w \in \{W_1, W_2, b_1, b_2\}$, with probability $1 - 8nd\exp \left( \frac{- \epsilon^2 (1 - 2\tau)^2B}{2ndC^2}\right)$,
    \begin{align*}
        \left \| \frac{1}{1 - 2 \tau}\frac{\partial \frac{1}{B} \sum_{i = 1}^B l\left(\model[\weight](x_i), f_i + y_i \mod 2 \right)}{ \partial w} - \frac{\partial \frac{1}{B} \sum_{i = 1}^B l\left(\model[\weight](x_i), y_i \right)}{ \partial w} \right \|_2 \le \epsilon.
    \end{align*}
    Here $C = \max_{w \in \{W_1, W_2, b_1, b_2\}}\max_{x \in \{0,1\}^n, y \in \{0,1\}} \| \frac{\partial \model[\weight](x_i)}{\partial w} \|_{\infty}$.
\end{lemma}
\begin{proof}
    Denote $n_i = 2f_i - 1  \in \{-1,1\}$.
    The key observation of this proof is that for the MAE loss, it holds that 
    \begin{align*}
        \frac{\partial l\left(\model[\weight](x_i), f_i + y_i \mod 2 \right)}{\partial w} = n_i\frac{\partial l\left(\model[\weight](x_i), f_i  \right)}{\partial w}.
    \end{align*}

    Now consider any index $k$ of $w$ ($k$ can be two dimensional for a matrix), $n_i\frac{\partial l\left(\model[\weight](x_i), f_i  \right)}{\partial w_k}$ are bounded random variable in $[-C,C]$, then by Hoeffding's bound,  it holds for any $t > 0$
    \begin{align}\label{eq:hoeffding}
        &\pr \left( \left| \sum_{i = 1}^B n_i\frac{\partial l\left(\model[\weight](x_i), f_i  \right)}{\partial w_k} - \E\left(\sum_{i = 1}^B n_i\frac{\partial l\left(\model[\weight](x_i), f_i  \right)}{\partial w_k}\right) \right| \ge t\right)  \notag
        \\&\le 2\exp\left( -\frac{t^2}{2BC^2}\right).
    \end{align}

    Now we have
    \begin{align*}
        \sum_{i = 1}^B n_i\frac{\partial l\left(\model[\weight](x_i), f_i  \right)}{\partial w_k} &= \frac{\partial \frac{1}{B} \sum_{i = 1}^B l\left(\model[\weight](x_i), f_i + y_i \mod 2 \right)}{\partial w_k}. \\
        \E\left[\sum_{i = 1}^B n_i\frac{\partial l\left(\model[\weight](x_i), f_i  \right)}{\partial w_k}\right] &= (1 - 2\tau) \frac{\partial \frac{1}{B} \sum_{i = 1}^B l\left(\model[\weight](x_i), y_i  \right)}{\partial w_k}.
    \end{align*}

    By~\Cref{eq:hoeffding}, by choosing $t = (1 - 2\tau)\frac{B\epsilon}{\sqrt{nd}}$, we would have with probability $1 - 8nd\exp \left( \frac{- \epsilon^2 (1 - 2\tau)^2B}{2ndC^2}\right) $, it holds that 
    \begin{align*}
        \left \| \frac{1}{1 - 2 \tau}\frac{\partial \frac{1}{B} \sum_{i = 1}^B l\left(\model[\weight](x_i), f_i + y_i \mod 2 \right)}{ \partial w} - \frac{\partial \frac{1}{B} \sum_{i = 1}^B l\left(\model[\weight](x_i), y_i \right)}{ \partial w} \right\|_2 \le \epsilon.
    \end{align*}
    The proof is then complete.
\end{proof}

\subsection{Generalization Effect}
\label{sec:theory_generalization}

In modern deep learning theory, explaining the generalization effect of neural network is a long standing open problem. 
Under~\Cref{setting:abundant}, the generalization gap is naturally zero as the data distribution coincides with the population distribution.
However, under~\Cref{setting:moderate} and~\Cref{setting:restricted}, it would be necessary for us to consider generalization effect.
As we observed in~\Cref{sec:experiment}, under~\Cref{setting:moderate}, the generalization gap is still small and the key difficulty lies in the optimization 
of the neural network. Under~\Cref{setting:restricted}, however, the gap between training accuracy and test accuracy is typically large.
This result is expected because as the sample complexity decreases, the optimization landscape over population distribution and data distribution tends 
to differ. Our experiment shows that applying L2 regularization, or equivalently, weight decay, helps to mitigate this problem and reduce the learning rate.

Although the optimization dynamics of neural networks over the general LPN problem is hard to track, existing literature contains results showing the provable benefit of weight decay on sample complexity under a special setting. We include this result in this subsection for completeness. It remains an open problem to extend this result to secret $s$ with hamming weight proportional to the problem dimension.

\begin{theorem}[Informal Version of Theorem 1.1 in~\cite{wei2019regularization}]
\label{thm:wd}
     If a gradient-based optimization algorithm uses SGD optimizer, logistic loss, MLP with depth $1$ initialized by Kaiming initialization, weight decay $0$, and a small enough learning rate, the sample complexity required to learn the secret of a parity problem with dimension $n$ and Hamming weight $2$ secret is $\Omega(n^2)$.
    However, with proper weight decay constant $\lambda > 0$, the sample complexity of the same problem could be reduced to $\tilde O(n)$.
\end{theorem}

\subsection{Discussion on the Hardness of Using Neural Networks to Solve LPN}
\label{sec:theory_hardness}

As mentioned in~\Cref{sec:theory_optimization}, it is proved in~\cite{abbe2020poly,shalev2017failures} that it requires $n^{\Omega(k)}$ sample complexity to solve the parity problem without noise using the full batch gradient descent method, which is also the regime our~\Cref{thm:optimization} falls into. 

However, this hardness constraint fails to hold for the stochastic gradient descent method. In fact~\cite{abbe2020poly} show that if a family of distribution is polynomial-time learnable, then there exists a polynomial-size neural network architecture, with a polynomial-time computable initialization that only depends on the family of the distribution, such that when stochastic gradient descent with batch size $1$ is performed on the network, it can learn the distribution in polynomial time complexity. We would like to remark some implications of these results.
\begin{itemize}
    \item These results suggest that whether there exists an architecture and initialization scheme such that~\Cref{alg:abundant_theory} can solve the LPN problem in polynomial time remains open and is inherently equivalent to whether LPN is in P.
    \item The construction of the neural network that simulates the polynomial time learning algorithm given by~\cite{abbe2020poly} relies heavily on deterministic initialization and it conjectures that SGD on random initialization will still require super polynomial time or sample complexity. 
    \item Although the full-batch gradient method will require $n^{\Omega(k)}$ time and sample complexity to solve the parity problem, this does not exclude the gradient method from being used to solve the LPN problem. First, it is still possible to use the gradient method as a building block of a larger algorithm that can effectively solve LPN. Second, given the best exponent component is not known, it is still possible that neural networks will show supreme performance over other classical algorithms, especially in the medium dimension with high noise regime, which our~\Cref{alg:moderate} tries to address.
\end{itemize}

\bibliography{lpn_ref}

\end{document}